\documentclass[letterpaper]{article}
\usepackage{proceed2e}
\usepackage[margin=1in]{geometry}

\usepackage{times}
 
\title{Stochastic Learning for Sparse Discrete Markov Random Fields with Controlled Gradient Approximation Error}

\newcommand*\samethanks[1][\value{footnote}]{\footnotemark[#1]}
\author{ {\bf Sinong Geng \thanks{\quad Sinong Geng and Zhaobin Kuang contribute equally. Their names are listed in alphabetical order.  Corresponds to: sgeng2@wisc.edu.}} \\
UW-Madison\\
\And
{\bf Zhaobin Kuang \samethanks[1]}  \\
UW-Madison\\
\And
{\bf Jie Liu}  \\
University of Washington\\
\And
{\bf Stephen Wright} \\
UW-Madison\\
\And
{\bf David Page} \\
UW-Madison\\
}

\usepackage{enumitem}
\newenvironment{prettyitem}[1]{
\vspace{-0.5mm}
\begin{itemize}[leftmargin=#1]
\setlength \itemsep{-0.3em}}
{\end{itemize}}

\usepackage{amsmath,amsthm,amssymb}
\allowdisplaybreaks
\usepackage{caption}
\usepackage{subcaption}

\usepackage{bm}
\usepackage{mathrsfs}
\usepackage{capt-of}

\usepackage{hyperref}
\usepackage{natbib}
\usepackage{subcaption}

\newcommand{\grad}{\bm{\nabla}}

\newcommand{\norm}[1]{\lVert #1 \rVert}
\newcommand{\Norm}[1]{\left\lVert #1 \right\rVert}
\newcommand{\curly}[1]{\left\{#1\right\}}

\newcommand{\btheta}{\bm{\theta}}
\newcommand{\bTheta}{\bm{\Theta}}
\newcommand{\bx}{\mathbf{x}}

\newcommand{\bX}{\mathbf{X}}
\newcommand{\bY}{\mathbf{Y}}

\newcommand{\bdelta}{\bm{\delta}}
\newcommand\given[1][]{\:#1\vert\:}

\newtheorem{lemma}{Lemma}
\newtheorem{theorem}{Theorem}
\newtheorem{definition}{Definition}

\usepackage{mathtools}
\DeclarePairedDelimiter\abs{\lvert}{\rvert}%
\makeatletter
\let\oldabs\abs
\def\abs{\@ifstar{\oldabs}{\oldabs*}}

\makeatother

\usepackage{algorithm}
\usepackage{algorithmicx}
\usepackage{algpseudocode}
\algrenewcommand\algorithmicindent{0.8em}

\usepackage{graphicx}

\usepackage{multirow}
\usepackage{float}

\usepackage[usenames]{color}

\usepackage[none]{hyphenat}
\binoppenalty=\maxdimen
\relpenalty=\maxdimen

\begin{document}

\maketitle

\begin{abstract}
We study the $L_1$-regularized maximum likelihood estimator/estimation (MLE) problem for discrete Markov random fields (MRFs), where efficient and scalable learning requires both sparse regularization and approximate inference. To address these challenges, we consider a stochastic learning framework called stochastic proximal gradient (SPG; \citealt{honorio2012convergence, atchade2014stochastic, miasojedow2016sparse}). SPG is an \emph{inexact} proximal gradient algorithm \citep{schmidt2011convergence}, whose inexactness stems from the stochastic oracle (Gibbs sampling) for gradient approximation -- exact gradient evaluation is infeasible in general due to the NP-hard inference problem for discrete MRFs \citep{koller2009probabilistic}. Theoretically, we provide novel \emph{verifiable} bounds to inspect and control the quality of gradient approximation. Empirically, we propose the \emph{tighten asymptotically} (TAY) learning strategy based on the verifiable bounds to boost the performance of SPG.
\end{abstract}

\section{INTRODUCTION}
Markov random fields (MRFs, a.k.a.~Markov networks, undirected graphical models) are a compact representation of the joint distribution among multiple variables, with each variable being a node and an edge between two nodes indicating conditional dependence between the two corresponding variables. Sparse discrete MRF learning is proposed in the seminal work of \cite{lee2006efficient}. By considering an $l_1$-regularized MLE problem, many components of the parameterization are driven to zero, yielding a sparse solution to structure learning. However, in general, solving an $l_1$-regularized MLE problem exactly for a discrete MRF is infamously difficult due to the NP-hard inference problem posed by exact gradient evaluation \citep{koller2009probabilistic}. We hence inevitably have to compromise accuracy for the gain of efficiency and scalability via \emph{inexact} learning techniques \citep{liu2013bayesian, liu2014learning, liu2016multiple,geng2018temporal}.

In this paper, we consider stochastic proximal gradient (SPG; \citealt{honorio2012convergence, atchade2014stochastic, miasojedow2016sparse}), a stochastic learning framework for $l_1$-regularized discrete MRFs. SPG hinges on a stochastic oracle for gradient approximation of the log-likelihood function (inexact inference). However, both the theoretical guarantees and the practical performances of existing algorithms are unsatisfactory.  

The stochastic oracle behind SPG is Gibbs sampling \citep{levin2009markov}, which is an effective approach to draw samples from an intractable probability distribution. With enough samples, the intractable distribution can be approximated effectively by the empirical distribution, and hence many quantities (e.g., the gradient of the log-likelihood function) related to the intractable distribution can be estimated efficiently. Since SPG uses Gibbs sampling for gradient approximation, it can be viewed as an inexact proximal gradient method \citep{schmidt2011convergence}, whose success depends on whether the gradient approximation error can be effectively controlled. While previous works \citep{honorio2012convergence, atchade2014stochastic, miasojedow2016sparse} have shown that the quality of the gradient approximation can be improved \emph{in the long run} with increasingly demanding computational resources, such long term guarantees might not translate to satisfactory performance in practice (see Section~\ref{sec:exp}). Therefore, it is desirable to estimate and control the gradient approximation error of SPG meticulously in each iteration so that a more refined approximation to the exact gradient will be rewarded with a higher gain of efficiency and accuracy in practice.

Careful analysis and control of the quality of the gradient approximation of SPG call for the cross-fertilization of theoretical and empirical insights from stochastic approximate inference \citep{bengio2009justifying,fischer2011bounding}, inexact proximal methods \citep{schmidt2011convergence}, and statistical sampling \citep{mitliagkas2017improving}. Our contributions are hence both theoretical and empirical. Theoretically, we provide novel \emph{verifiable} bounds (Section~\ref{sec:bound}) to inspect and control the gradient approximation error induced by Gibbs sampling. Also, we provide a proof sketch for the main results in Section~\ref{sec:proof-sketch}. Empirically, we propose the \emph{tighten asymptotically} (TAY) learning strategy (Section~\ref{sec:applications}) based on the verifiable bounds to boost the performance of SPG.

\section{BACKGROUND}

We first introduce $l_1$-regularized discrete MRFs in Section~\ref{sec:mrf}. We then briefly review SPG as a combination of proximal gradient for sparse statistical learning and Gibbs sampling for addressing the intractable exact gradient evaluation problem. 


\subsection{$l_1$-Regularized Discrete MRF}
\label{sec:mrf}
For the derivation, we focus on the binary pairwise case and we illustrate that our framework can be generalized to other models in Section ~\ref{sec:applications}.
Let  $\bX = \begin{bmatrix}
X_1,X_2,\cdots,X_p \end{bmatrix}^\top \in \curly{0,1}^p$
be a $p\times 1$ binary random vector. We use an uppercase letter such as $X$ to denote a random variable and the corresponding lowercase letter to denote a particular \emph{assignment} of the random variable, i.e., $X = x$. We also use boldface letters to represent vectors and matrices and regular letters to represent scalars. We define the function $\bm{\psi}:\curly{0,1}^p \rightarrow \curly{0,1}^m, \; \bx \rightarrow \bm{\psi}(\bx)$ to represent the \emph{sufficient statistics} (a.k.a.~\emph{features}) whose values depend on the assignment $\bx$ and compose an $m\times 1$ vector $\bm{\psi}(\bx)$, with its $j^{th}$ component denoted as $\psi_j (\bx)$. We use $\mathbb{X}$ to represent a dataset with $n$ independent and identically distributed (i.i.d.) samples.

With the notation introduced above, the $l_1$-regularized discrete MRF problem can be formulated as the following convex optimization problem:
\begin{equation}
\begin{gathered}
\label{eq:l1mrf}
\hat{\btheta} = \arg\min_{\btheta\in\bTheta} -\frac{1}{n}\sum_{\mathbf{x}\in \mathbb{X}} \btheta^\top \bm{\psi}(\mathbf{x}) + A(\btheta) + \lambda \norm{\btheta}_1,
\end{gathered}
\end{equation}
\begin{figure}[H]
\begin{minipage}[t]{.475\textwidth}
\begin{algorithm}[H]
\caption{Gibbs Sampling (Gibbs-1)}
\label{alg:Gibbs-1}	
\begin{algorithmic}[1]
\Require initial samples $\mathbb{S}_0$ and $\btheta$.
\Ensure $\mathbb{S}$.
\Function{Gibbs-1}{$\mathbb{S}_0$, $\btheta$}
\State $\mathbb{S} \leftarrow \mathbb{S}_0$, and decide $p$ from $\mathbb{S}_0$.
\For{$i \in \left\{ 1, \cdots, p\right \}$}
\For{$\bx \in \mathbb{S}$}
\State Compute $\text{P}_{\btheta}(X_i\given \bx_{-i})$ according to (\ref{eq:cond}).
\State Update $x_i$ by $\text{P}_{\btheta}(X_i\given \bx_{-i})$.
\EndFor
\EndFor
\State \Return $\mathbb{S}$.
\EndFunction
\end{algorithmic}
\end{algorithm}
\end{minipage}
\hfill
  \hfill
  \begin{minipage}[t]{.475\textwidth}
\begin{algorithm}[H]
\caption{Gradient Approximation (GRAD)}
\label{alg:grad}	
\begin{algorithmic}[1]
\Require $\btheta$, $\mathbb{E}_{\mathbb{X}}\bm{\psi}(\bx)$, and $q$.
\Ensure $\bm{\Delta} f(\btheta)$.
\Function{Grad}{$\btheta$, $\mathbb{E}_{\mathbb{X}}\bm{\psi}(\bx)$, $q$}
\State Initialize $\mathbb{S}$ with $q$ samples.
\While{true}
\State $\mathbb{S} \leftarrow$ \Call{Gibbs-1}{ $\mathbb{S}$, $\btheta$}.
\If{stopping criteria met}\label{step:stop}
\State Compute $\mathbb{E}_{\mathbb{S}}\bm{\psi}(\bx)$ according to (\ref{eq:replacement}).
\State $\bm{\Delta} f(\btheta) \leftarrow \mathbb{E}_{\mathbb{S}}\bm{\psi}(\bx)-\mathbb{E}_{\mathbb{X}}\bm{\psi}(\bx)$. 
\State\textbf{break}.
\EndIf
\EndWhile
\State \Return $\bm{\Delta} f(\btheta)$.
\EndFunction
\end{algorithmic}
\end{algorithm}
  \end{minipage}%
  \hfill
  \hfill
  \begin{minipage}[t]{.475\textwidth}
    \begin{algorithm}[H]
\caption{Stochastic Proximal Gradient (SPG)}
\label{alg:pxgb}
\begin{algorithmic}[1]
\Require $\mathbb{X}$, $\lambda$, and $q$.
\Ensure $\tilde{\btheta}$.
\Function{SPG}{$\mathbb{X}$, $\lambda$, $q$}
\State Compute $\mathbb{E}_{\mathbb{X}}\bm{\psi}(\bx)$ according to (\ref{eq:der-moment}).
\State Initialize $\btheta^{(0)}$ randomly and $k \leftarrow 0$.
\State Choose step length $\alpha$. 
\While{true}
\State\label{step:pcd}$\bm{\Delta} f(\btheta^{(k)})\leftarrow$ \Call{Grad}{$\btheta^{(k)}$, $\mathbb{E}_{\mathbb{X}}\bm{\psi}(\bx)$, $q$}.
\State $\btheta^{(k+1)} \leftarrow \bm{\mathcal{S}}_{\alpha\lambda}\left(\btheta^{(k)} - \alpha \bm{\Delta} f(\btheta^{(k)})\right).$
\If{Stopping criteria met}\label{step:converge}
\State $\tilde{\btheta} =\btheta^{(k+1)}$, \Return $\tilde{\btheta}$.
\EndIf
\State $k \leftarrow k+1$
\EndWhile
\EndFunction
\end{algorithmic}
\end{algorithm}
  \end{minipage}
\end{figure}
with
\begin{equation*}
\begin{gathered}
 A(\btheta) = \log \hspace{-4mm} \sum_{\mathbf{x}\in \curly{0,1}^p} \hspace{-3mm} \exp(\btheta^\top \bm{\psi(\bx)}),
\end{gathered}
\end{equation*}
where $\bTheta\subseteq \mathbb{R}^m$ is the parameter space of $\btheta$'s, $\lambda \ge 0$, and $A(\btheta)$ is the \emph{log partition function}. We denote the differentiable part of (\ref{eq:l1mrf}) as
\begin{equation}
\label{eq:f}
f(\btheta) = -\frac{1}{n}\sum_{\bx\in \mathbb{X}} \btheta^\top \bm{\psi}(\bx) + A(\btheta).
\end{equation}
Solving (\ref{eq:l1mrf})  requires evaluating the gradient of $f(\btheta)$, which is given by:
\begin{equation}
\begin{gathered}
\label{eq:grad}
\grad f(\btheta) = \mathbb{E}_{\btheta} \bm{\psi}(\bx)-\mathbb{E}_{\mathbb{X}} \bm{\psi}(\bx),
\end{gathered}
\end{equation}
with
\begin{equation}
\begin{gathered}
\label{eq:der-moment}
\mathbb{E}_{\btheta} \bm{\psi}(\bx) = \hspace{-4mm} \sum_{\bx\in\curly{0,1}^p} \hspace{-3mm} \mathrm{P}_{\btheta}(\bx)\bm{\psi}(\bx), \quad \mathbb{E}_{\mathbb{X}} \bm{\psi}(\bx) =  \frac{1}{n}\sum_{\bx \in \mathbb{X}} \bm{\psi}(\bx).
\end{gathered}
\end{equation}
$\mathbb{E}_{\btheta} \bm{\psi}(\bx)$ represents the expectation of the sufficient statistics under $\mathrm{P}_{\btheta}(\bx)= \frac{\exp (\btheta^\top\bm{\psi}(\bx))}{\exp(A(\btheta))}$, which is a discrete MRF probability distribution parameterized by $\btheta$. $\mathbb{E}_{\mathbb{X}} \bm{\psi}(\bx)$ represents the expectation of the sufficient statistics under the empirical distribution. Computing $\mathbb{E}_{\mathbb{X}} \bm{\psi}(\bx)$ is straightforward, but computing $\mathbb{E}_{\btheta} \bm{\psi}(\bx)$ exactly is intractable due to the entanglement of $A(\btheta)$. As a result, various approximations have been made \citep{wainwright2007high, hofling2009estimation, viallon2014empirical}. 


\subsection{Stochastic Proximal Gradient}
\label{sec:SPG}

To efficiently solve (\ref{eq:l1mrf}), many efforts have been made in combining Gibbs sampling \citep{levin2009markov} and proximal gradient descent \citep{parikh2014proximal} into SPG, a method that adopts the proximal gradient framework to update iterates, but uses Gibbs sampling as a stochastic oracle to approximate the gradient when the gradient information is needed \citep{honorio2012convergence,atchade2014stochastic,miasojedow2016sparse}.

Specifically, Gibbs sampling with $q$ chains running $\tau$ steps (Gibbs-$\tau$) can generate $q$ samples for $\mathrm{P}_{\btheta}(\bx)$. Gibbs-$\tau$ is achieved by iteratively applying Gibbs-$1$ for $\tau$ times. Gibbs-$1$ is summarized in Algorithm~\ref{alg:Gibbs-1}, where 
\begin{equation}
\label{eq:cond}
\text{P}_{\btheta}(X_i \mid \mathbf{x}_{-i}) = \text{P}_{\theta}(\bx_i \given x_{1}, \cdots, x_{i-1}, x_{i+1}, \cdots, x_{p})
\end{equation}
represents the conditional distribution of $X_i$ given the assignment of the remaining variables $\mathbf{x}_{-i}$ under the parameterization $\btheta$. Denoting the set of these $q$ (potentially repetitive) samples as $\mathbb{S}$, we can approximate $\mathbb{E}_{\btheta} \bm{\psi}(\bx)$ by the easily computable 
\begin{equation}
\begin{gathered}
\label{eq:replacement}
\mathbb{E}_{\mathbb{S}} \bm{\psi}(\bx)=\frac{1}{q} \sum_{\bx \in \mathbb{S}} \bm{\psi}(\bx)
\end{gathered}
\end{equation}
and thus reach the approximated gradient $\bm{\Delta} f(\btheta) = \mathbb{E}_{\mathbb{S}}\bm{\psi}(\bx)-\mathbb{E}_{\mathbb{X}}\bm{\psi}(\bx)$ with the gradient approximation error:
\begin{equation*} 
\bdelta(\btheta) = \bm{\Delta} f(\btheta)-\grad f(\btheta). 
\end{equation*} 
By replacing $\grad f(\btheta)$ with $\bm{\Delta} f(\btheta)$ in proximal gradient, the update rule for SPG can be derived as $\btheta^{(k+1)} = \bm{\mathcal{S}}_{\alpha\lambda}\left(\btheta^{(k)}-\alpha \bm{\Delta} f(\btheta^{(k)})\right)$,
where $\alpha>0$ is the step length and $\bm{\mathcal{S}}_{\lambda}(\bm{a})$ is the soft-thresholding operator whose value is also an $m \times 1$ vector, with its $i^{th}$ component defined as $\mathcal{S}_{\lambda}(\bm{a})_i = \mathrm{sgn}(a_i)\max(0, \abs{a_i} - \lambda)$ and $\mathrm{sgn}(a_i)$ is the sign function.

By defining 
\begin{align}
\label{eq:G}
\begin{split}
\bm{G}_\alpha(  \btheta^{(k)}) := &\frac{1}{\alpha} \left(\btheta^{(k)} - \btheta^{(k+1)}\right)
\\=  &\frac{1}{\alpha} \left(\btheta^{(k)} - S_{\alpha\lambda} \left(\btheta^{(k)}-\alpha \bm{\Delta} f(\btheta^{(k)})\right)\right),
\end{split}
\end{align}
we can rewrite the previous update rule in a form analogous to the update rule of a standard gradient descent, resulting in the update rule of a \emph{generalized gradient descent} algorithm:
\begin{equation}
\label{eq:ggd}
\btheta^{(k+1)} = \btheta^{(k)} - \alpha \bm{G}_\alpha(\btheta^{(k)}).
\end{equation}
SPG is summarized in Algorithm~\ref{alg:pxgb}. Its gradient evaluation procedure based on Algorithm~\ref{alg:Gibbs-1} is given  in Algorithm~\ref{alg:grad}.

\section{MOTIVATION}
\label{sec:convergence-rate-err}
Both practical performance and theoretical guarantees of SPG are still far from satisfactory. Empirically, there are no convincing schemes for selecting $\tau$ and $q$, which hinders the efficiency and accuracy of SPG. Theoretically, to the best of our knowledge, existing non-asymptotic convergence rate guarantees can only be achieved for SPG with an averaging scheme \citep{schmidt2011convergence, honorio2012convergence, atchade2014stochastic} (see also Section~\ref{sec:convergence-rate}), instead of ordinary SPG. In contrast, in the exact proximal gradient descent method, the objective function value is non-decreasing and convergent to the optimal value under some mild assumptions \citep{parikh2014proximal}. In Section~\ref{sec:decreasing-objective}, we identify that the absence of non-asymptotic convergence rate guarantee for SPG primarily comes from the existence of gradient approximation error $\bdelta(\btheta) $. In Section~\ref{sec:convergence-rate}, we further validate that the objective function value achieved by SPG is also highly dependent on $\bdelta(\btheta) $. These issues bring about the demand of inspecting and controlling $\bdelta(\btheta) $ in each iteration.

\subsection{Setup and Assumptions}
For the ease of presentation, we rewrite the objective function in (\ref{eq:l1mrf}) as $g(\btheta) = f(\btheta)+h(\btheta)$,
where $h(\btheta) = \lambda \Norm{\btheta}_1$, and $f(\btheta)$ is given in (\ref{eq:f}). Since $\grad f(\btheta)$ is Lipschitz continuous \citep{honorio2012lipschitz}, we denote its Lipschitz constant as $L$. We also make the same assumption that $\alpha \le 1/L$ as \cite{schmidt2011convergence}.


\subsection{Decreasing Objective}
\label{sec:decreasing-objective}
It is well-known that exact proximal gradient enjoys a $O\left(\frac{1}{k}\right)$ convergence rate \citep{parikh2014proximal}. One premise for this convergence result is that the objective function value decreases in each iteration. However, satisfying the decreasing condition is much more intricate in the context of SPG. Theorem~\ref{thm:obj-decrease} clearly points out that $\bdelta(\btheta)$ is one main factor determining whether the objective function decreases in SPG.  

\begin{theorem}\normalfont
\label{thm:obj-decrease}
Let $\btheta^{(k)}$ be the iterate of SPG after the $k^{th} $ iteration. Let $\btheta^{(k+1)}$ be defined as in (\ref{eq:ggd}). With $\alpha \le 1/L$, we have
\begin{align*}
\begin{split}
g(\btheta^{(k+1)}) - g(\btheta^{(k)}) \le 
\alpha \bdelta (\btheta^{(k)})^\top \bm{G}_\alpha &(\btheta^{(k)})
\\- &\frac{\alpha}{2}\norm{\bm{G}_\alpha(\btheta^{(k)})}_2^2.
\end{split}
\end{align*}
Furthermore, a sufficient condition  for $g(\btheta^{(k+1)})< g(\btheta^{(k)})$ is 
\begin{equation*}
\norm{\bdelta(\btheta^{(k)})}_2 < \frac{1}{2}\norm{\bm{G}_\alpha(\btheta^{(k)})}_2.
\end{equation*}
\end{theorem}

According to Theorem~\ref{thm:obj-decrease}, if the magnitude of the noise, quantified by $\norm{\bdelta(\btheta^{(k)})}_2$, is reasonably small, the objective function value decreases in each iteration. Under this condition, we can further construct a theoretical support for the convergence rate of the objective function value in the Section~\ref{sec:convergence-rate}. 

\subsection{Convergence Rate}
\label{sec:convergence-rate}

Assuming that $\bdelta(\btheta)$ is small enough in each iteration to generate a decreasing objective value sequence, we can derive  Theorem~\ref{thm:convergence-rate} following Proposition 1 in \cite{schmidt2011convergence}:

\begin{theorem}\normalfont
\label{thm:convergence-rate}
Let $\mathcal{K}=(\btheta^{(0)},\btheta^{(1)},\btheta^{(2)},\cdots,\btheta^{(\kappa)})$ be the iterates generated by Algorithm~\ref{alg:pxgb}. Then if $g(\btheta^{(k+1)}) \le g(\btheta^{(k)})$ with $k \in \{1, 2, \cdots, \kappa-1 \}$, we have
\begin{align}
\label{eq:con-avg}
\begin{split}
g (\btheta^{(\kappa)}) - g(\hat{\btheta}) \le &\\
\frac{L}{2\kappa}& \left(\norm{\btheta^{(0)}-\hat{\btheta}}_2 + \frac{2}{L}\sum_{k=1}^\kappa \norm{\bdelta(\btheta^{(k)})}_2 \right)^2.
\end{split}
\end{align}
\end{theorem}
Recall that $\hat{\btheta}$ is an optimal solution to the sparse MLE problem defined in (\ref{eq:l1mrf}). From (\ref{eq:con-avg}), it is obvious that if the gradient approximation error is reasonably small, then during the early iterations of SPG, $\norm{\btheta^{(0)}-\hat{\btheta}}_2$ dominates $\frac{2}{L}\sum_{k=1}^\kappa \norm{\bdelta(\btheta^{(k)})}_2$.  Therefore, in the beginning, the convergence rate is $O(1/\kappa)$. However, as the iteration proceeds, $\frac{2}{L}\sum_{k=1}^\kappa \norm{\bdelta(\btheta^{(k)})}_2$ accumulates and hence in practice SPG can only maintain a convergence rate of $O(1/\kappa)$ up to some noise level that is closely related to $\bdelta(\btheta^{(k)})$. Therefore, $\bdelta(\btheta^{(k)})$ plays an importance role in the performance of SPG.

Notice that Theorem~\ref{thm:convergence-rate} offers convergence analysis of the objective function value in the last iteration $g (\btheta^{(\kappa)})$. This result is different from the existing non-asymptotic analysis on $g(\sum_{k=1}^{\kappa} \btheta^{(k)}/{\kappa})$, the objective function evaluated on the average of all the visited solutions \citep{schmidt2011convergence, honorio2012convergence, atchade2014stochastic}. Theorem~\ref{thm:convergence-rate} is more practical than previous analysis, since $\sum_{k=1}^{\kappa} \btheta^{(k)}/{\kappa}$ is a dense parameterization not applicable to structure learning.

According to the analysis above, we need to control $\bdelta(\btheta^{(k)})$ in each iteration to achieve a decreasing and $O\left( \frac{1}{k} \right)$-converging objective function value sequence. Therefore, we focus on checkable bounds for gradient approximation error in Section~\ref{sec:bound}.

\section{MAIN RESULTS}
\label{sec:bound}
In this section, we derive an asymptotic and a non-asymptotic bound to control the gradient approximation error $\bdelta(\btheta^{(k)})$ in each iteration. For this purpose, we consider an arbitrary $\btheta$, and perform gradient approximation via Gibbs-$\tau$ using Algorithm~\ref{alg:grad}, given an initial value for the Gibbs sampling algorithm, $\tilde{\bx}_0$. By bounding $\bdelta(\btheta)$, we can apply the same technique to address $\bdelta(\btheta^{(k)})$. 

We first provide a bound for the magnitude of the conditional expectation of $\bdelta(\btheta)$, $\Norm{\mathbb{E}_{\tilde{\bx}_\tau} [\bdelta(\btheta)\mid \tilde{\bx}_0]}_2$, in Section~\ref{sec:asy-bound}. Based on this result, we further draw a non-asymptotic bound for the magnitude of the gradient approximation error, $\Norm{\bdelta(\btheta)}_2$, in Section~\ref{sec:non-asy-bound}. Both results are \emph{verifiable} in each iteration. 

For the derivation of the conclusions, we will focus on binary pairwise Markov networks (BPMNs). Let $\bx\in\curly{0,1}^p$ and $\btheta$ be given, a binary pairwise Markov network \citep{hofling2009estimation, geng2017efficient} is defined as: 
\begin{equation}
\label{eq:bpmn}
\mathrm{P}_{\btheta} (\bx) = \frac{1}{Z(\btheta)} \exp \left(\sum_{i=1}^p\sum_{j\ge i}^p\theta_{ij}x_ix_j\right),
\end{equation}
where $Z(\btheta)=\exp(A(\btheta))$ is the partition function. $\theta_{ij}$ is a component of $\btheta$ that represents the strength of conditional dependence between $X_i$ and $X_j$. 

\subsection{An Asymptotic Bound}
\label{sec:asy-bound}
We first consider the magnitude of the conditional expectation of $\bdelta(\btheta)$ with respect to $\tilde{\bx}_\tau$, $\Norm{\mathbb{E}_{\tilde{\bx}_\tau} [\bdelta(\btheta)\mid \tilde{\bx}_0]}_2$. To this end, we define $\mathbf{U}$ a $p\times p$ \emph{computable} matrix that is related to $\btheta$ and the type of MRF in question. $U_{ij}$, the component in the $i^{th}$ row and the $j^{th}$ column of $\mathbf{U}$, is defined as follows:
\begin{align}
\label{eq:U}
U_{ij} =   \frac{\lvert \exp\left(-\xi_{ij}\right) - 1 \rvert b^*}{\left(1+b^*\exp\left(-\xi_{ij}\right)\right)(1+b^*)}, 
\end{align}
where
\begin{gather*}
b^* =  \max\curly{r,\min\curly{s,\exp\left(\frac{\xi_{ij}}{2}\right)}},\\
s =  \exp \left( - \xi_{ii} -\sum_{k\ne i,k\ne j}\xi_{ik} \max\curly{-\mathrm{sgn}(\xi_{ik}),0}  \right)，\\
r = \exp \left(- \theta_{ii} - \sum_{k\ne i,k\ne j} \xi_{i, k} \max\curly{\mathrm{sgn}(\xi_{i, k}),0}   \right),
\end{gather*}
and $\mathrm{sgn}(\xi_{ik})$ is the sign function evaluated on $\xi_{ij} = \theta_{\min\curly{i,j},\max\curly{i,j}}$.

We then define $\mathbf{B}_i$ as a $p\times p$ identity matrix except that its $i^{th}$ row is replaced by the $i^{th}$ row of $\mathbf{U}$, with $i\in\curly{1,2,\cdots,p}$. We further define 
\begin{equation*}
\mathbf{B} = \mathbf{B}_p \mathbf{B}_{p-1} \mathbf{B}_{p-2} \cdots \mathbf{B}_i \cdots\mathbf{B}_1
\end{equation*}
and the grand sum $\mathscr{G}(\mathbf{B}) = \sum_{i=1}^p\sum_{j=1}^p  B_{ij}$, where $B_{ij}$ is the entry in the $i^{th}$ row and the $j^{th}$ column of $\mathbf{B}$. With the definitions above, $\Norm{\mathbb{E}_{\tilde{\bx}_\tau} [\bdelta(\btheta)\mid \tilde{\bx}_0]}_2$ can be upper bounded by Theorem~\ref{thm:exp-err-bound}.

\begin{theorem}\label{thm:exp-err-bound}\normalfont
Let $\tilde{\bx}_\tau$ be the sample generated after running Gibbs sampling for $\tau$ steps (Gibbs-$\tau$) under the parameterization $\btheta$ initialized by $\tilde{\bx}_0\in\curly{0,1}^p$; then with $m$ denoting the size of sufficient statistics, the following inequality holds:
\begin{equation}
\label{eq:exp-err-bound}
\Norm{\mathbb{E}_{\tilde{\bx}_\tau} [\bdelta(\btheta)\mid \tilde{\bx}_0]}_2 \le  2\sqrt{m} \mathscr{G}(\mathbf{B}^\tau),
\end{equation}
where $\mathbf{B}^\tau$ represents the $\tau^{th}$ power of $\mathbf{B}$.
\end{theorem}

In Theorem~\ref{thm:exp-err-bound}, the bound provided is not only observable in each iteration, but also efficient to compute, offering a convenient method to inspect the quality of the gradient approximation. When the spectral norm of $\mathbf{U}$ is less than $1$, the left hand side of (\ref{eq:exp-err-bound}) will converge to 0. Thus, by increasing $\tau$, we can decrease $\Norm{\mathbb{E}_{\tilde{\bx}_\tau}[\bdelta(\btheta)\mid \tilde{\bx}_0]}_2$ to an arbitrarily small value. 

Theorem~\ref{thm:exp-err-bound} is derived by bounding the influence of a variable on another variable in $\bX$ (i.e., the Dobrushin influence defined in \ref{def:d-influence-matrix}) with $\mathbf{U}$. Furthermore, $\mathbf{U}$ defined in (\ref{eq:U}) is a sharp bound of the Dobrushin influence whenever $b^* \neq \exp\left(\frac{\xi_{ij}}{2}\right)$, explaining why (\ref{eq:exp-err-bound}) using the definition of $\mathbf{U}$ is tight enough for practical applications. 

\subsection{A Non-Asymptotic Bound}
\label{sec:non-asy-bound}
In order to provide a non-asymptotic guarantee for the quality of the gradient approximation, we need to concentrate $\Norm{\bdelta(\btheta)}_2$ around $\Norm{\mathbb{E}_{\tilde{\bx}_\tau} [\bdelta(\btheta) \mid \tilde{\bx}_0]}_2$. Let $q$ defined in Section~\ref{sec:SPG} be given. Then, $q$ trials of Gibbs sampling are run, resulting in $q$ samples, $\{\tilde{\bx}_\tau^{(1)}, \tilde{\bx}_\tau^{(2)}, \cdots, \tilde{\bx}_\tau^{(q)}\}$. That is to say, for each sufficient statistic, $\psi_j(\btheta)$, with $j \in \{1, 2, \cdots, m\}$, we have $q$ samples, $\curly{\psi_j^{(1)}(\btheta) , \psi_j^{(2)}(\btheta) ,\cdots, \psi_j^{(q)}(\btheta)}$. Defining the sample variance of the corresponding sufficient statistics as $V_{\psi_j}$, we have Theorem~\ref{thm:sample-err-bound} to provide a non-asymptotic bound for $\Norm{\bdelta(\btheta)}_2$: 

\begin{theorem}\normalfont
\label{thm:sample-err-bound}
Let $\btheta$, $q$, and an arbitrary $\tilde{\bx}_0\in\curly{0,1}^p$ be given. Let $m$ represent the dimension of $\btheta$ and $\Norm{\bdelta(\btheta)}_2$ represent the magnitude of the gradient approximation error by running $q$ trials of Gibbs-$\tau$ initialized by  $\tilde{\bx}_0$. Compute $\mathbf{B}$ according to Section~\ref{sec:asy-bound} and choose $\epsilon_j>0$. Then, with probability at least $1 - 2 \sum_{j = 1}^m \beta_j$, where $\beta_j > 0$, $j\in\curly{1,2,\cdots,m}$,
\begin{equation}
\begin{gathered}
\label{eq:sample-err-bound}
\Norm{ \bdelta(\btheta)}_2 \le 2\sqrt{m}\left(\mathscr{G}(\mathbf{B}^\tau)+\sqrt{\frac{\sum_{ j = 1}^m \epsilon_j^2}{4m}}\right),
\end{gathered}
\end{equation}
with $\beta_j$ satisfying
\begin{equation}
\begin{gathered}
\epsilon_j =2 \left(\sqrt{\frac{V_{\psi_j}\ln 2/\beta_j } {2q}} + \frac{7\ln2/\beta_j}{3(q-1)}\right). 
\end{gathered}
\end{equation}
\end{theorem}

Notice that the bound in Theorem~\ref{thm:sample-err-bound} is easily \emph{checkable}, i.e., given $\tau$, $q$, $V_{\psi_j}$'s, and $\btheta$, we can determine a bound for $\Norm{ \bdelta(\btheta)}_2$ that holds with high probability. Furthermore, Theorem~\ref{thm:sample-err-bound} provides the sample complexity needed for gradient estimation. Specifically, given small enough $\beta_j$'s, if we let 
\begin{equation*}
\mathscr{G}(\mathbf{B}^\tau) = \sqrt{\sum_{ j = 1}^m \epsilon_j^2/4m},
\end{equation*}
we can show that
\begin{equation*} 
 2\sqrt{m}\left(\mathscr{G}(\mathbf{B}^\tau)+\sqrt{\sum_{ j = 1}^m \epsilon_j^2/4m}\right)= O \left (\frac{1}{q}\right ).
\end{equation*} 
That is to say, by assuming that $\mathscr{G}(\mathbf{B}^\tau)$ and  $\sqrt{\sum_{ j = 1}^m \epsilon_j^2/4m}$ share the same scale, the upper bound of the gradient approximation error converges to 0 as $q$ increases. Moreover, we include sample variance, $V_{\psi_j}$'s, in (\ref{eq:sample-err-bound}). This is because the information provided by sample variance leads to an improved data dependent bound.  

\section{PROOF SKETCH OF MAIN RESULTS}
\label{sec:proof-sketch}
As mentioned in Section~\ref{sec:non-asy-bound}, the non-asymptotic result in  Theorem~\ref{thm:sample-err-bound} is derived from the asymptotic bound in Theorem~\ref{thm:exp-err-bound} by concentration inequalities, we therefore only highlight the proof of Theorem~\ref{thm:exp-err-bound} in this section, and defer other technical results to Supplements. Specifically, the proof of Theorem~\ref{thm:exp-err-bound} is divided into two parts: bounding $\Norm{\mathbb{E}_{\tilde{\bx}_\tau} [\bdelta(\btheta)\mid \tilde{\bx}_0]}_2$ by the total variation distance (Section~\ref{sec:bound-expectation-tv}) and bounding the total variation distance (Section~\ref{sec:bound-TV}). 

\subsection{Bounding \texorpdfstring{$\Norm{\mathbb{E}_{\tilde{\bx}_\tau} [\bdelta(\btheta)\mid \tilde{\bx}_0]}_2$}{TEXT} by the Total Variation Distance}
\label{sec:bound-expectation-tv}

To quantify $\Norm{\mathbb{E}_{\tilde{\bx}_\tau} [\bdelta(\btheta)\mid \tilde{\bx}_0]}_2$, we first introduce the concept of total variation distance \citep{levin2009markov} that measures the distance between two distributions over $\curly{0,1}^p$.
\begin{definition}\normalfont
Let $u(\bx)$, and $v(\bx)$ be two probability distributions of $\bx\in\curly{0,1}^p$. Then the total variation distance between $u(\bx)$ and $v(\bx)$ is given as:
\begin{equation*}
\Norm{u(\bx)-v(\bx)}_{\text{TV}} = \frac{1}{2} \hspace{-2mm} \sum_{\bx\in\curly{0,1}^p} \hspace{-2mm} \lvert u(\bx)-v(\bx) \rvert.
\end{equation*}
\end{definition}

With the definition above, $\Norm{\mathbb{E}_{\tilde{\bx}_\tau} [\bdelta(\btheta)\mid \tilde{\bx}_0]}_2$ can be upper bounded by the total variation distance between two distributions ($\mathrm{P}_\tau(\bx\mid\tilde{\bx}_0)$ and $\mathrm{P}_{\btheta}(\bx)$) using the following lemma:

\begin{lemma}\label{lem:exp-err}\normalfont
Let $\tilde{\bx}_\tau$ be the sample generated after running Gibbs sampling for $\tau$ steps (Gibbs-$\tau$) under the parameterization $\btheta$ initialized by $\tilde{\bx}_0\in\curly{0,1}^p$, then the following is true:
\begin{equation*}
\Norm{\mathbb{E}_{\tilde{\bx}_\tau} [\bdelta(\btheta)\mid \tilde{\bx}_0]}_2 \le  2\sqrt{m}\Norm{\mathrm{P}_\tau(\bx\mid\tilde{\bx}_0) -\mathrm{P}_{\btheta}(\bx)}_{\text{TV}}.
\end{equation*}
\end{lemma}

With Lemma~\ref{lem:exp-err}, bounding $\Norm{\mathbb{E}_{\tilde{\bx}_\tau} [\bdelta(\btheta)\mid \tilde{\bx}_0]}_2$ can be achieved by bounding the total variation distance $\Norm{\mathrm{P}_\tau(\bx\mid\tilde{\bx}_0) -\mathrm{P}_{\btheta}(\bx)}_{\text{TV}}$. Recent advances in the quality control of Gibbs samplers offer us \emph{verifiable} upper bounds for $\Norm{\mathrm{P}_\tau(\bx\mid\tilde{\bx}_0) -\mathrm{P}_{\btheta}(\bx)}_{\text{TV}}$ on the learning of a variety of MRFs \citep{mitliagkas2017improving}. However, they can not be applied to BPMNs because of the positivity constraint on parameters. We describe these next. 

\subsection{Bounding  \texorpdfstring{$\Norm{\mathrm{P}_\tau(\bx\mid\tilde{\bx}_0) -\mathrm{P}_{\btheta}(\bx)}_{\text{TV}}$}{TEXT}}

\label{sec:bound-TV}

Now we generalize the analysis in \cite{mitliagkas2017improving} to BPMNs without constraints on the sign of parameters by introducing the definition of the Dobrushin influence matrix and a technical lemma.
\begin{definition}[Dobrushin influence matrix]\normalfont
\label{def:d-influence-matrix}
The Dobrushin influence matrix of $\mathrm{P}_{\btheta}(\bx)$ is a $p\times p$ matrix $\mathbf{C}$ with its component in the $i^{th}$ row and the $j^{th}$ column, $C_{ij}$, representing the influence of $X_j$ on $X_i$ given as:
\begin{equation*}
C_{ij} = \max_{(\bX,\bY) \in N_j} \Norm{\mathrm{P}_{\btheta}(X_i\mid\bX_{-i})-\mathrm{P}_{\btheta}(Y_i\mid\bY_{-i})}_{\text{TV}},
\end{equation*}
where $(\bX,\bY) \in N_j$ represents $X_l=Y_l$ for all $l\ne j$.
\end{definition}

\begin{lemma}\normalfont
\label{lem:u-bpmn}
Let $\mathrm{P}_{\btheta} (\bx)$ represent a binary pairwise Markov network defined in (\ref{eq:bpmn}) that is parameterized by $\btheta$. An upper bound of the total influence matrix is given by $\mathbf{U}$ defined in Section~\ref{sec:asy-bound}.
\end{lemma}

It should be noticed that, similar to the Theorem 12 in \cite{mitliagkas2017improving}, Lemma~\ref{lem:u-bpmn} provides an exact calculation except when $b^* =\exp\left(\frac{\xi_{i, j}}{2}\right)$. 

Therefore, we can consider the $\mathbf{U}$ defined in Section~\ref{sec:asy-bound} as an upper bound for Dobrushin influence matrix in BPMN and thus apply $\mathbf{U}$ to Theorem 9 in \cite{mitliagkas2017improving}. Then, we have
\begin{equation*}
\Norm{\mathrm{P}_\tau(\bx\mid\tilde{\bx}_0) -\mathrm{P}_{\btheta}(\bx)}_{\text{TV}} \le \mathscr{G}(\mathbf{B}^\tau),
\end{equation*}
where $\mathbf{B}^\tau$ represents the $\tau^{th}$ power of $\mathbf{B}$. Theorem~\ref{thm:exp-err-bound} follows this combined with Lemma~\ref{lem:exp-err}

\section{STRUCTURE LEARNING}
\label{sec:applications}
With the two bounds introduced in Section~\ref{sec:bound}, we can easily examine and control the quality of gradient approximation in each iteration by choosing $\tau$. In detail, we introduce a criterion for the selection of $\tau$ in each iteration. Satisfying the proposed criterion, the objective function is guaranteed to decrease asymptotically. That is to say, the difference between $g(\btheta^{(k+1)})$ and $g(\hat{\btheta})$ is asymptotically \emph{tightened}, compared with the difference between $g\left(\btheta^{(k)}\right)$ and $g(\hat{\btheta})$. Therefore, we refer to the proposed criterion as \ref{eq:criterion-aggressive}. Furthermore, using \ref{eq:criterion-aggressive} we provide an improved SPG method denoted by TAY for short.

Specifically, staring from $\tau = 1$, TAY stops increasing $\tau$ when the following bound is satisfied:  
\begin{equation}
\label{eq:criterion-aggressive}
2\sqrt{m} \mathscr{G}(\mathbf{B}^\tau) < \frac{1}{2}\norm{\bm{G}_\alpha(\btheta^{(k)})}_2.  \tag{\textsc{TAY-Criterion}}
\end{equation}

We can also derive a non-asymptotic counterpart of \ref{eq:criterion-aggressive} by combining the results of Theorem~\ref{thm:obj-decrease} and Theorem~\ref{thm:sample-err-bound}: 
\begin{equation}
\begin{gathered}
\label{eq:criterion-conservative}
0<2\sqrt{m}\left( \mathscr{G}(\mathbf{B}^\tau)+\sqrt{\frac{\sum_{ j = 1}^m \epsilon_j^2}{4m}}\right) \le \frac{1}{2}\norm{\bm{G}_\alpha(\btheta^{(k)})}_2, \\ \epsilon_j =2 \left(\sqrt{\frac{2V_{\psi_j}\ln 2/\beta_j } {4q}} + \frac{7\ln2/\beta_j}{3(q-1)}\right),
\end{gathered}
\end{equation}
where the $V_{\psi_j}$'s and $\beta_j$'s are defined in Theorem~\ref{thm:sample-err-bound}. (\ref{eq:criterion-conservative}) provides the required sample complexity, $q$, for TAY in each iteration. However, the selection of $q$ according to (\ref{eq:criterion-conservative}) is conservative, because it includes the worst-case scenario where the gradient approximation errors in any two iterations cannot offset each other.

In Section~\ref{sec:tay-criterion} and \ref{sec:tay}, we theoretically analyze the performance guarantees of \ref{eq:criterion-aggressive} and the convergence of TAY, respectively. 
\subsection{ Guarantees of \ref{eq:criterion-aggressive}}
\label{sec:tay-criterion}
The theorem below provides the performance guarantee for \ref{eq:criterion-aggressive} in each iteration. 
\begin{theorem}\normalfont
\label{thm:criterion-decrease}
Let $\btheta^{(k)}$ and $\tilde{\bx}_0$ be given. Let $q$ and $\mathbf{B}$ defined in Theorem~\ref{thm:sample-err-bound} be given. For $\btheta^{(k+1)}$ generated in Algorithm~\ref{alg:pxgb} using \ref{eq:criterion-aggressive}, the following is true:
\begin{align*}
\lim_{q \to \infty}\mathrm{P}\left(g(\btheta^{(k+1)}) < g(\btheta^{(k)}) \given 
\quad\quad\quad\quad\quad\quad\quad \right.\\
\left. \quad2\sqrt{m} \mathscr{G}(\mathbf{B}^\tau) < \frac{1}{2}\norm{\bm{G}_\alpha(\btheta^{(k)})}_2   \right) = 1.
\end{align*}
\end{theorem}
Theorem~\ref{thm:criterion-decrease} makes a statement that the objective function value decreases with large $q$. Specifically, \ref{eq:criterion-aggressive} assumes that the upper bound of the conditional expectation of $\Norm{\bdelta(\btheta)}_2$ is small enough to satisfy the sufficient condition proven in Theorem~\ref{thm:obj-decrease}. When the number of samples $q$ is large enough, $\Norm{\bdelta(\btheta)}_2$ itself is very likely to meet the condition and hence the objective function is also likely to decrease with \ref{eq:criterion-aggressive} satisfied. 

\subsection{ Convergence of TAY}
\label{sec:tay}
Finally, based on Theorem~\ref{thm:convergence-rate} and Theorem~\ref{thm:criterion-decrease}, we derive the following theorem on the convergence of TAY.  

\begin{theorem}\normalfont
\label{thm:convergence-rate-tay}
Let $\mathcal{K}=(\btheta^{(0)},\btheta^{(1)},\btheta^{(2)},\cdots,\btheta^{(\kappa)})$ be the iterates generated by TAY. Then, with $k \in \{1, 2, \cdots, \kappa-1 \}$, the following is true:
\begin{minipage}{\linewidth}
\begin{align*}
\lim_{q \to \infty}&\mathrm{P}\left[g (\btheta^{(\kappa)}) - g(\hat{\btheta}) \le \vphantom{\frac{L}{2\kappa} \left(\norm{\btheta^{(0)}-\hat{\btheta}}_2 + \frac{2}{L}\sum_{k=1}^\kappa \norm{\bdelta(\btheta^{(k)})}_2 \right)^2} \right.\\
&\left.\frac{L}{2\kappa} \left(\norm{\btheta^{(0)}-\hat{\btheta}}_2 + \frac{2}{L}\sum_{k=1}^\kappa \norm{\bdelta(\btheta^{(k)})}_2 \right)^2
  \right] = 1,
\end{align*}
\end{minipage}
where $\hat{\btheta}$ is defined in (\ref{eq:l1mrf}).
\end{theorem}

\begin{figure*}[t]
\centering
\begin{subfigure}{0.32\textwidth}
\centering
\includegraphics[scale=0.20]{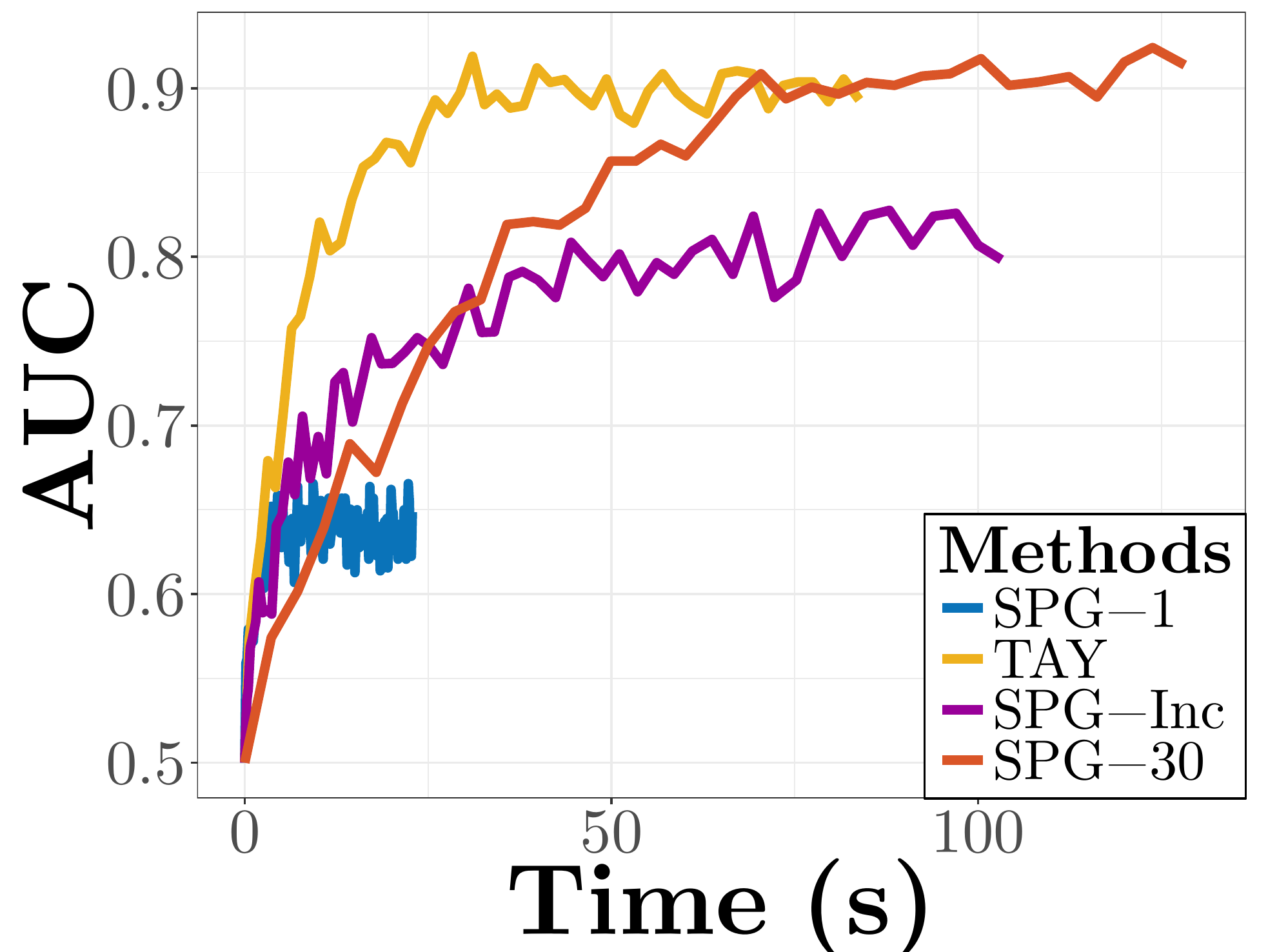}
\caption{AUC v.s. Time \\10 nodes}\label{fig:time1}
\end{subfigure}
\begin{subfigure}{0.32\textwidth}
\centering
\includegraphics[scale=0.20]{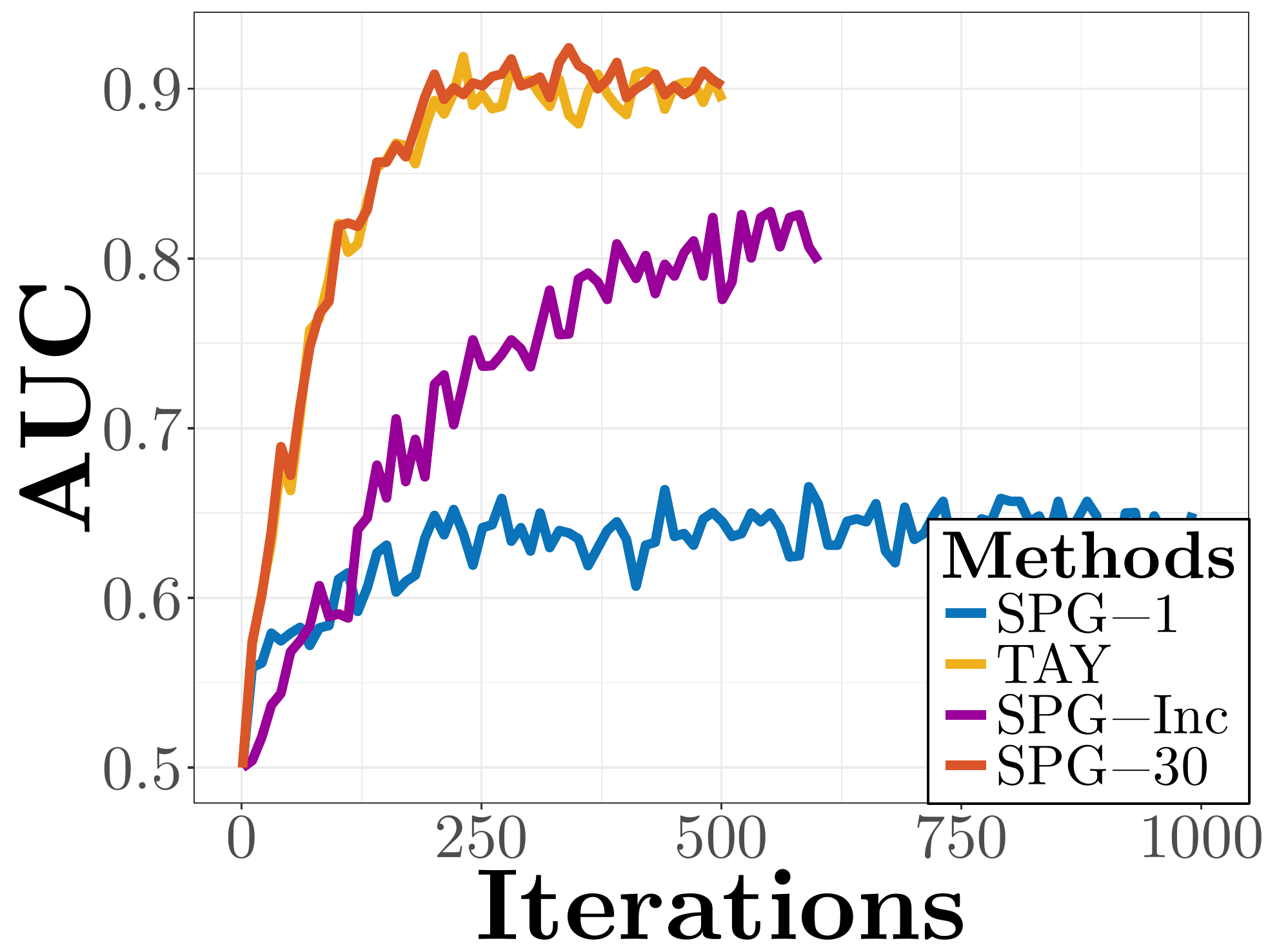}
\caption{AUC v.s. Iterations\\ 10 nodes}\label{fig:iterations1}
\end{subfigure}
\begin{subfigure}{0.32\textwidth}
\centering
\includegraphics[scale=0.20]{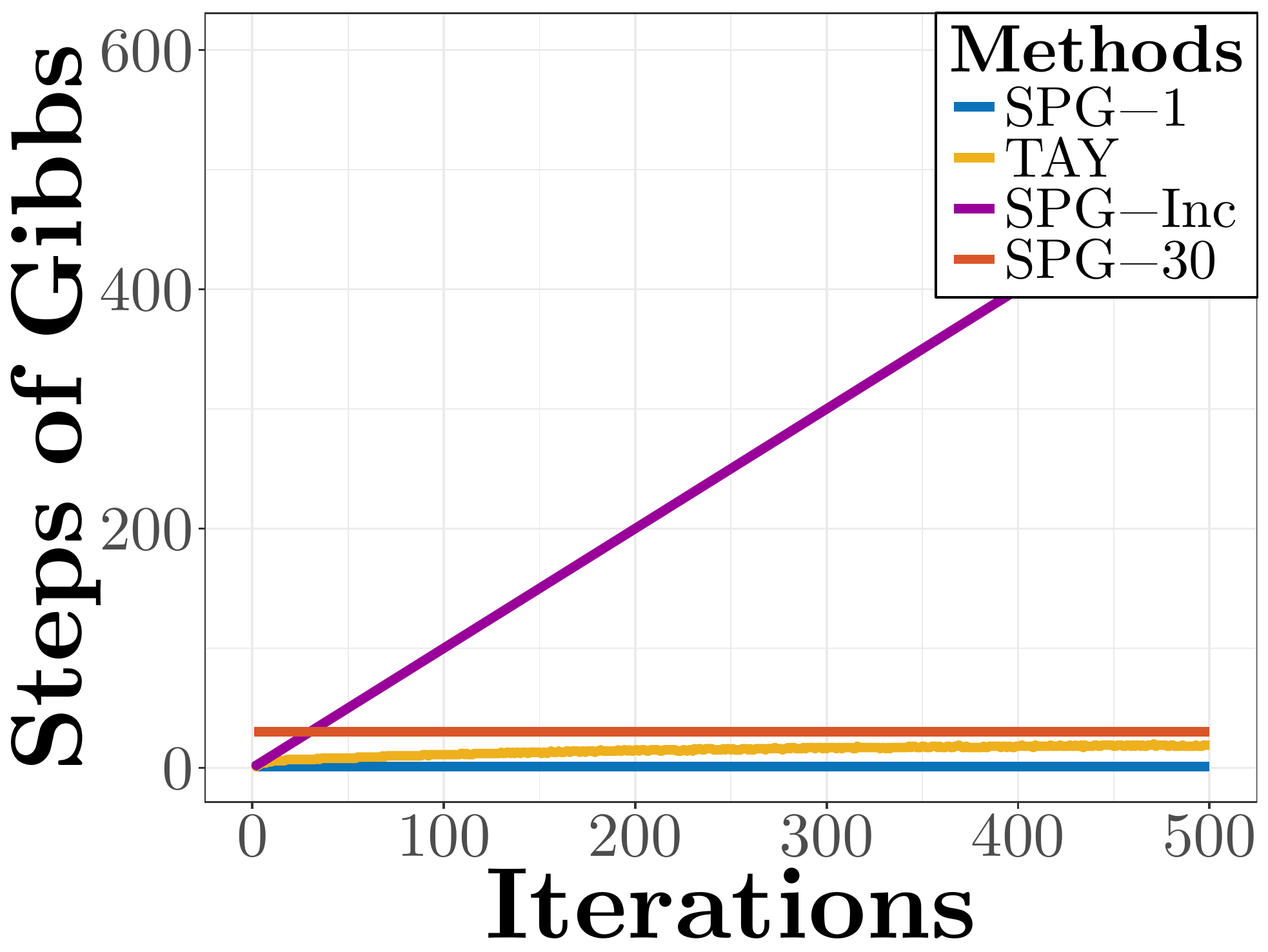}
\caption{$\tau$ v.s. Iterations \\ 10 nodes}
\end{subfigure}
\centering
\begin{subfigure}{0.32\textwidth}
\centering
\includegraphics[scale=0.20]{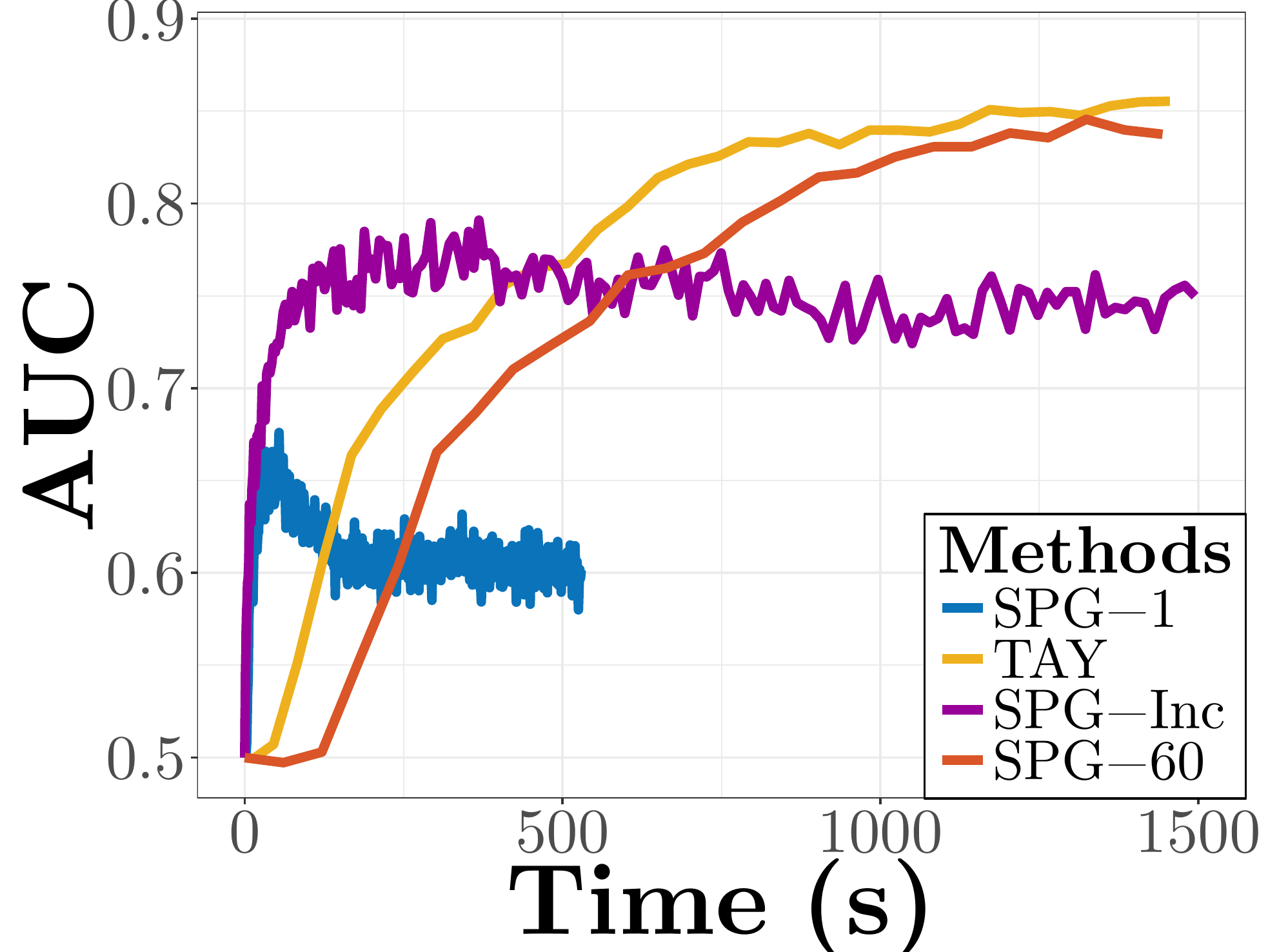}
\caption{AUC v.s. Time \\20 nodes}\label{fig:time2}
\end{subfigure}
\centering
\begin{subfigure}{0.32\textwidth}
\centering
\includegraphics[scale=0.20]{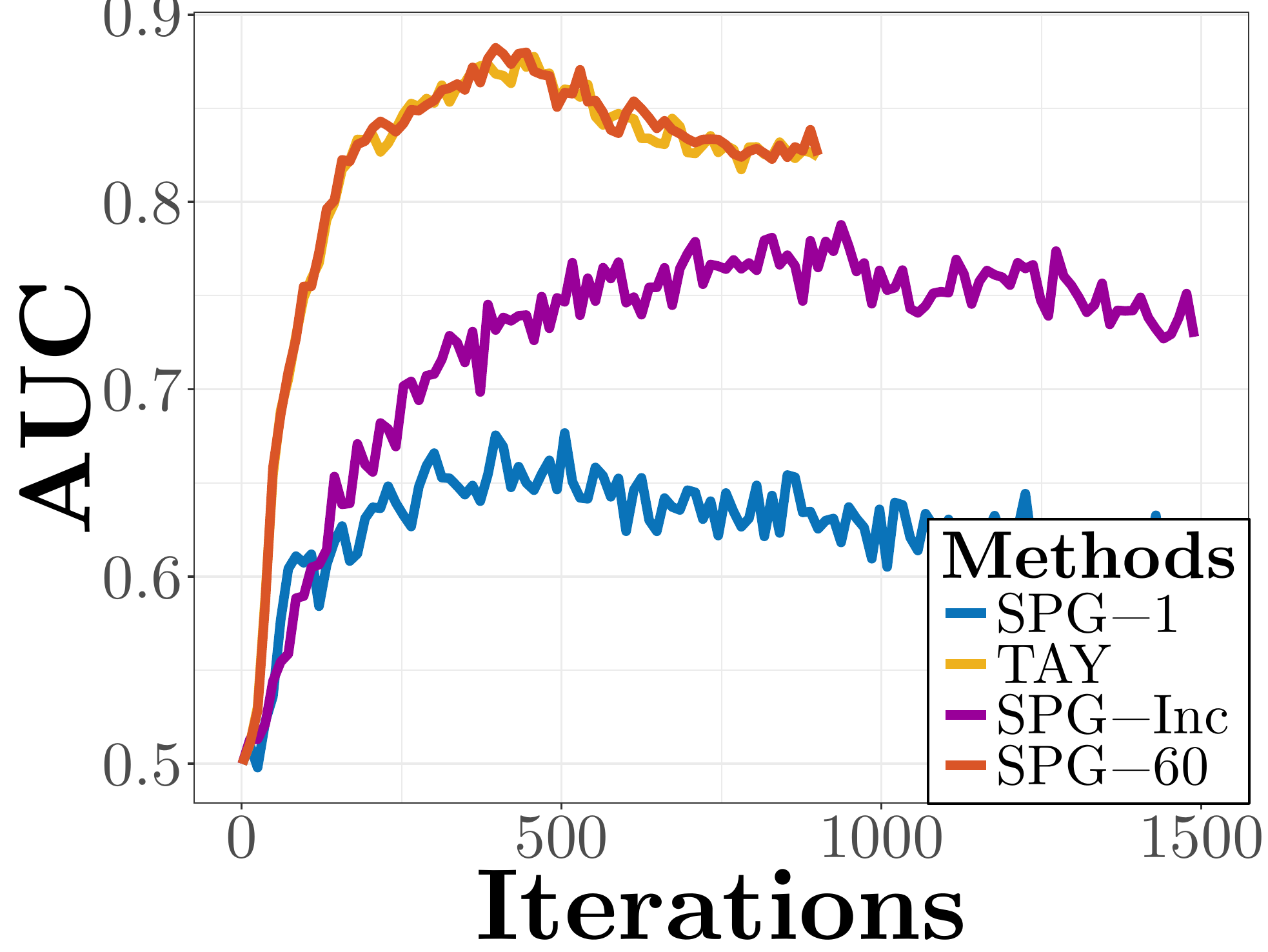}
\caption{AUC v.s. Iterations\\ 20 nodes}\label{fig:iterations2}
\end{subfigure}
\centering
\begin{subfigure}{0.32\textwidth}
\centering
\includegraphics[scale=0.20]{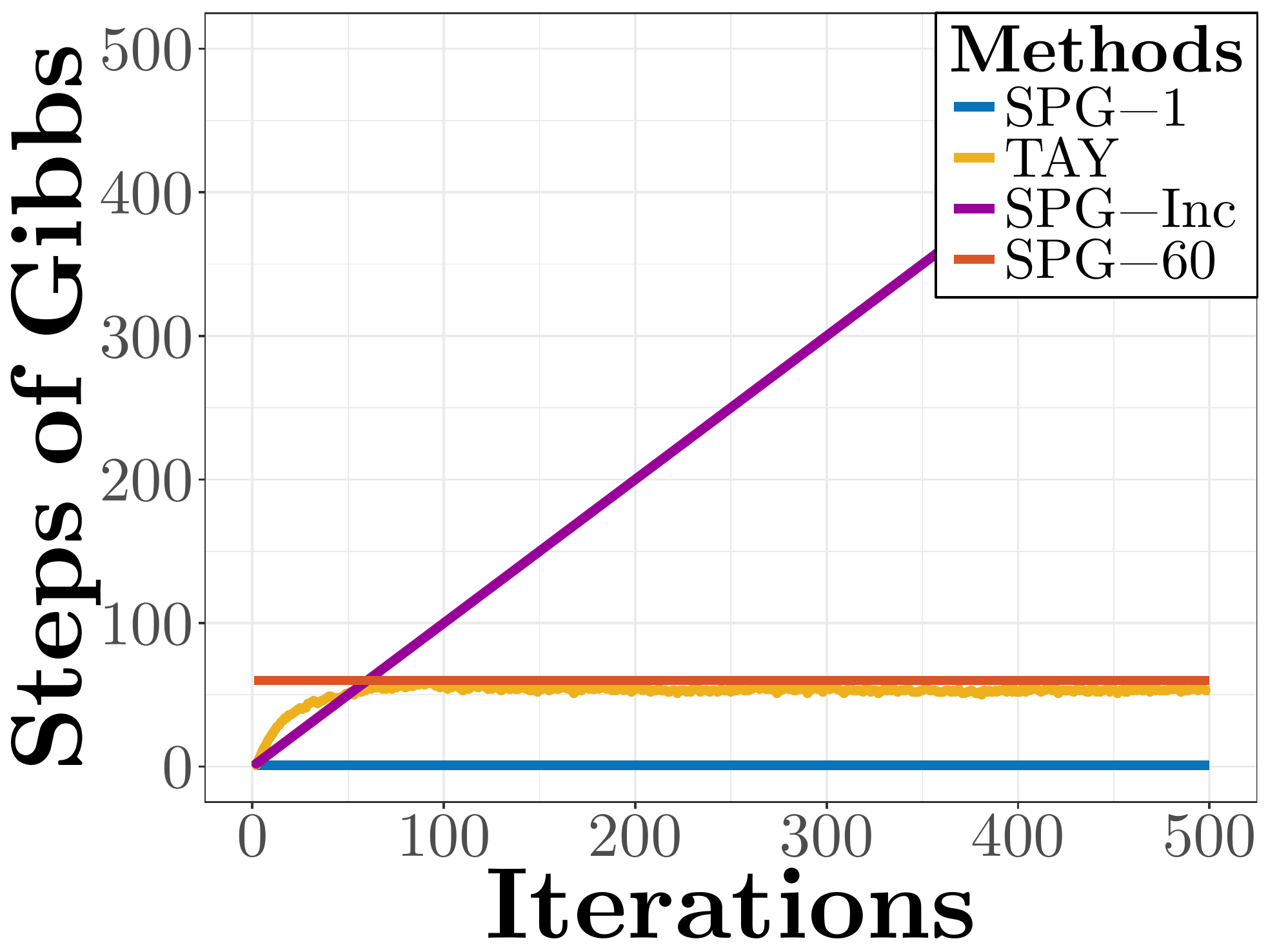}
\caption{$\tau$ v.s. Iterations \\ 20 nodes}
\end{subfigure}
\vspace{-0.3cm}
\caption{Area under curve (AUC) and the steps of Gibbs sampling ($\tau$) in each iteration for structure learning of a 20-node network.}
\label{fig:structure-learning}
\end{figure*}

\subsection{Generalizations}
As we demonstrate in Section~\ref{sec:bound} and Section~\ref{sec:proof-sketch}, the derivation of our main results relies on bounding the Dobrushin influence with $\mathbf{U}$ and we show a procedure to construct $\mathbf{U}$ in the context of BPMNs. Moreover, \cite{mitliagkas2017improving} and \cite{liu2014projecting} provide upper bounds $\mathbf{U}$'s for other types of discrete pairwise MRFs. Therefore, combined with their results, our framework can also be applied to other discrete pairwise Markov networks. Dealing with pairwise MRFs is without any loss of generality, since any discrete MRF can be transformed into a pairwise one \citep{wainwright2008graphical, ravikumar2010high}.

\section{EXPERIMENTS}
\label{sec:exp}
We demonstrate that the structure learning of discrete MRFs benefits substantially from the application of TAY with synthetic data and that the bound provided on the gradient estimation error by Theorem~\ref{thm:exp-err-bound} is tighter than existing bounds. To illustrate that TAY is readily available for practical problems, we also run TAY using a real world dataset.
Because of the limit of space, we only report the experiments under one set of representative experiment configurations. Exhaustive results using different experiment configurations are presented in the Supplements.

\subsection{Structure Learning}
\label{sec:structure-learning}
In order to demonstrate the utility of TAY for effectively learning the structures of BPMNs, we simulate two BPMNs (one with 10 nodes and the other one with 20 nodes):

\begin{prettyitem}{*}
\item We set the number of features to $p = 10$ ($p = 20$). Components of $\btheta$ in the ground truth model are randomly chosen to be nonzero with an edge generation probability of $0.3$. The non-zero components of the real parameter have a uniform distribution on $[-2,-1]\bigcup[1,2]$
\item 1000 (2000 for 20 nodes) samples are generated by Gibbs sampling with 1000 burn-in steps.
\item The results are averaged over 10 trials.
\end{prettyitem}

The sizes of the BPMNs generated in this paper are comparable to those in \citep{honorio2012convergence, atchade2014stochastic, miasojedow2016sparse}.
\begin{figure}[t]
\centering
\includegraphics[scale=0.25]{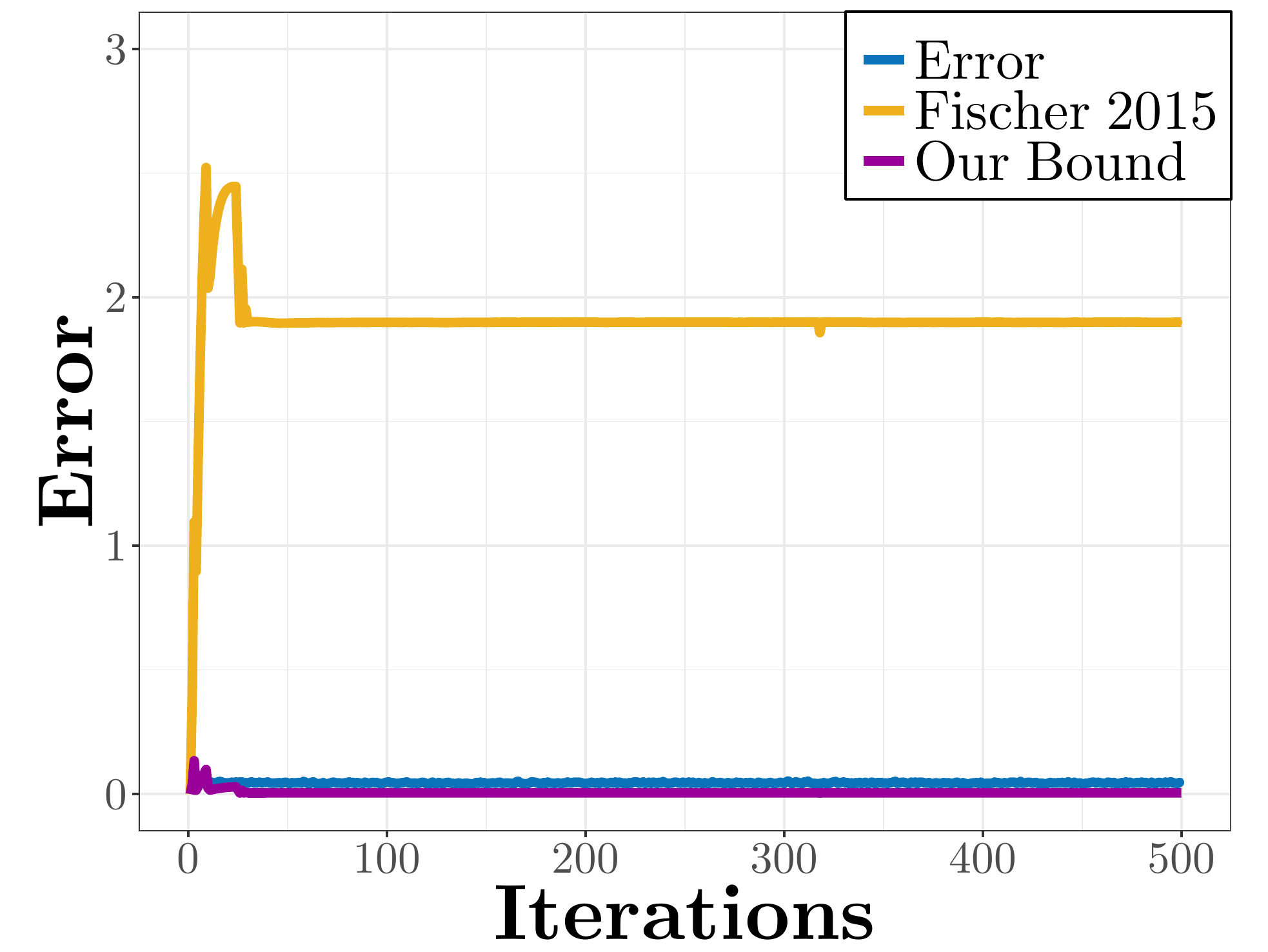}
\caption{The gradient approximation error, the existing bound and the bound (\ref{eq:exp-err-bound}) in the structure learning of a 10-node network.}\label{fig:tightness}
\end{figure}

\begin{figure*}[t]
\centering
\includegraphics[scale=0.5]{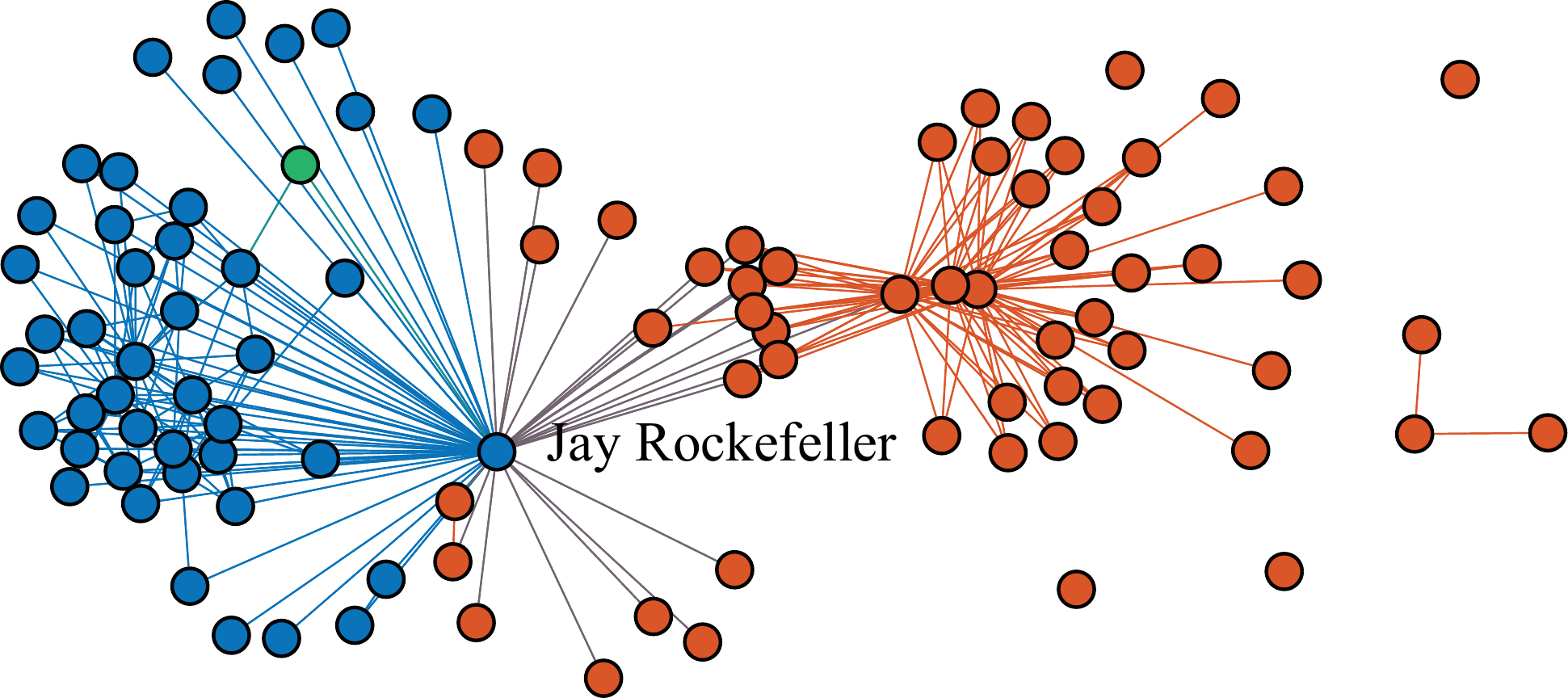}
\vspace{0.3cm}
\caption{The result of TAY on the senator voting data: Red vertices denote Republicans, blue Democracts, and green Independent. 
The figure is rendered by \texttt{Gephi} \citep{bastian2009gephi}.}\label{fig:senate} 
\end{figure*}
Then, using the generated samples, we consider SPG and TAY. According to the analysis in Section~\ref{sec:bound}, the quality of the gradient approximation is closely related to the number of Gibbs sampling steps $\tau$. However, for SPG, there are no convincing schemes for selecting $\tau$. We select $\tau = 30$ ($\tau = 60$ for 20 nodes) to ensure that the gradient approximation error is small enough. Furthermore, we also evaluate the performance of the algorithm using an increasing $\tau$ ( $\tau = k$ in the $k^{th}$ iteration), suggested by \cite{atchade2014stochastic} (SPG-Inc). 

To strike a fair comparison, we use the same step length $\alpha = 0.4$ and regularization parameter $\lambda = 0.025$ ( $\lambda = 0.017$ for 20 nodes) for different methods. We do not tune the step length individually for each method, since \cite{atchade2014stochastic} has shown that various learning rate selection schemes have minimal impact on the performance in the context of SPG. The number of chains used in Gibbs sampling, $q$, is not typically a tunable parameter either, since it indicates the allocation of the computational resources. For each method, it can be easily noticed that the larger the number of samples is, the slower but more accurate the method will be. Furthermore, if the $q$'s are different for different methods, it would be difficult to distinguish the effect of $\tau$ from that of $q$. Therefore, we set it to $2000$ for 10-node networks and $5000$ for 20-node networks. Performances of different methods are compared using the area under curve (AUC) of receiver operating characteristic (ROC) for structure learning in Figure~\ref{fig:structure-learning}. The Gibbs sampling steps in each method are also compared in Figure \ref{fig:structure-learning}. 

Notice that we plot AUCs against both time (Figure~\ref{fig:time1} and Figure~\ref{fig:time2}) and iterations (Figure~\ref{fig:iterations1} and Figure~\ref{fig:iterations2}). The two kinds of plots provide different information about the performances of different methods: the former ones focus on overall complexity and the latter illustrate iteration complexity. We run each method until it converges. Using much less time, TAY achieves a similar AUC to SPG with $\tau = 30$ and $\tau = 60$.
Moreover, SPG with $\tau = 1$ reaches the lowest AUC, since the quality of the gradient approximation cannot be guaranteed with such a small $\tau$. Therefore, the experimental results indicate that TAY adaptively chooses a $\tau$ achieving reasonable accuracy as well as efficiency for structure learning in each iteration. For a more thorough comparison, we also contrast the performance of TAY and a non-SPG-based method, i.e., the pseudo-likelihood method \citep{hofling2009estimation, geng2017efficient}, in the Supplements. As a result, the two methods achieve comparable AUCs.

\subsection{Tightness of the Proposed Bound}
According to the empirical results above, TAY needs a $\tau$ only on the order of ten, suggesting that the bound in Theorem~\ref{thm:exp-err-bound} is tight enough for practical applications. To illustrate this more clearly, we compare (\ref{eq:exp-err-bound}) with another bound on the expectation of the gradient approximation error derived by \cite{fischer2015training}.
Specifically, we calculate the gradient approximation error, the bound (\ref{eq:exp-err-bound}), and \cite{fischer2015training}'s bound, in each iteration of learning a 10-node network. The results are reported in Figure~\ref{fig:tightness}. Notice that the bound in \cite{fischer2015training} gets extraordinarily loose with more iterations. Considering this, we may need run Gibbs chains for thousands of steps if we use this bound. In contrast, bound (\ref{eq:exp-err-bound}) is close to and even slightly less than the real error. This is reflective of the fact that the proposed bound is on the expectation instead of the error itself. As a result, (\ref{eq:exp-err-bound}) is much tighter and thus more applicable.

\subsection{Real World Data}
In our final experiment, we run TAY using the Senate voting data from the second session of the $109^{th}$ Congress \citep{USSenate38:online}. The dataset has 279 samples and 100 variables. Each sample represents the vote cast by each of the 100 senators for a particular bill, where $0$ represents nay, and $1$ represents yea. Missing data are imputed as $0$'s. The task of interest is to learn a BPMN model that identifies some clusters that represent the dependency between the voting inclination of each senator and the party with which the senator is affiliated.

We use TAY with $ \alpha = 0.4$. 5000 Markov chains are used for Gibbs sampling. Since our task is exploratory analysis, $\lambda=0.1$ is selected in order to deliver an interpretable result. The proposed algorithm is run for 100 iterations. The resultant BPMN is shown in Figure~\ref{fig:senate}, where each node represents the voting record of a senator and the edges represent some positive dependency between the pair of senators connected. The nodes in red represent Republicans and the nodes in blue represents Democrats. The clustering effects of voting consistency within a party are captured, coinciding with conventional wisdom. More interestingly, Jay Rockefeller, as a Democrat, has many connections with Republicans. This is consistent with the fact that his family has been a ``traditionally Republican dynasty'' \citep{wiki:JayRocke32}.



\section{CONCLUSION}
We consider SPG for $l_1$-regularized discrete MRF estimation. Furthermore, we conduct a careful analysis of the gradient approximation error of SPG and provide upper bounds to quantify its magnitude. With the aforementioned analysis, we introduce a learning strategy called TAY and show that it can improve the accuracy and efficiency of SPG. 

\small \textbf{Acknowledgement}: Sinong Geng, Zhaobin Kuang, and David Page would like to gratefully acknowledge the NIH BD2K Initiative grant U54 AI117924 and the NIGMS grant 2RO1 GM097618. Stephen Wright would like to gratefully acknowledge NSF Awards IIS-1447449, 1628384, 1634597, and 1740707; AFOSR Award FA9550-13-1-0138; Subcontracts 3F-30222 and 8F-30039 from Argonne National Laboratory; and DARPA Award N660011824020.

\nocite{liu2013genetic,
liu2014new,
liu2016multiple}

\clearpage
\subsubsection*{References}
\renewcommand\refname{\vskip -1cm}
\bibliography{ref}

\clearpage\onecolumn
\appendix
\section*{Supplements}
\addcontentsline{toc}{section}{Appendices}
\section{Proofs}
\subsection{Proof of Theorem~\ref{thm:obj-decrease}}
\label{sec:proof-obj-decrease}
We first introduce the following technical lemma.

\begin{lemma}\label{lemma:ineq}\normalfont Let $g(\btheta)$, $f(\btheta)$, and $h(\btheta)$ be defined as in Section~\ref{sec:mrf}; hence $f(\btheta)$ is convex and differentiable, and $\grad f(\btheta)$ is Lipschitz continuous with Lipschitz constant $L$. Let $\alpha\le {1}/{L}$. Let $\bm{G}_\alpha(\btheta)$ and $\bm{\Delta} f(\btheta)$ be defined as in Section (\ref{sec:SPG}). Then for all $\btheta_1$ and $\btheta_2$, the following inequality holds:
\begin{align}
\begin{split}
\label{eq:lemma-ineq}
g(\btheta_1^\dag)& \le g(\btheta_2) + \bm{G}_{\alpha}^\top(\btheta_1)(\btheta_1-\btheta_2)+(\grad f(\btheta_1)-\bm{\Delta} f(\btheta_1))^\top (\btheta_1^\dag-\btheta_2)-\frac{\alpha}{2}\Norm{\bm{G}_\alpha (\btheta_1)}_2^2,
\end{split}
\end{align}\normalsize
where $\btheta_1^\dag = \btheta_1-\alpha \bm{G}_\alpha (\btheta_1)$.
\end{lemma}
\begin{proof}
The proof is based on the convergence analysis of the standard proximal gradient method \citep{vandenberghe16}. $f(\btheta)$ is a convex differentiable function whose gradient is Lipschitz continuous with Lipschitz constant $L$. By the quadratic bound of the Lipschitz property:
\begin{equation*}
f(\btheta_1^\dag) \le f(\btheta_1) - \alpha \grad^\top f(\btheta_1) \bm{G}_\alpha(\btheta_1) + \frac{\alpha^2L}{2}\Norm{\bm{G}_\alpha(\btheta_1)}_2^2.
\end{equation*}
With $\alpha \le {1}/{L}$, and adding $h(\btheta_1^\dag)$ on both sides of the quadratic bound, we have an upper bound for $g(\btheta_1^\dag)$:
\begin{align*}
\begin{split}
g(\btheta_1^\dag)  \le f(\btheta_1) - \alpha \grad^\top f(\btheta_1) \bm{G}_\alpha(\btheta_1)  + \frac{\alpha}{2}&\Norm{\bm{G}_\alpha(\btheta_1)}_2^2 + h(\btheta_1^\dag).
\end{split}
\end{align*}
By convexity of $f(\btheta)$ and $h(\btheta)$, we have:
\begin{gather*}
f(\btheta_1) \le f(\btheta_2)+\grad^\top f(\btheta_1) (\btheta_1-\btheta_2),\\
 h(\btheta_1^\dag) \le h(\btheta_2) + (\bm{G}_\alpha(\btheta_1)-\Delta f(\btheta_1))^\top (\btheta_1^+ - \btheta_2),
\end{gather*}
which can be used to further upper bound $g(\btheta_1^\dag)$, and results in (\ref{eq:lemma-ineq}). Note that we have used the fact that $G_\alpha(\btheta_1)-\Delta f(\btheta_1)$ is a subgradient of $h(\btheta_1^\dag)$ in the last inequality.
\end{proof}

With Lemma~\ref{lemma:ineq}, we are now able to prove Theorem~\ref{thm:obj-decrease}.  In Lemma~\ref{lemma:ineq}, let $\btheta_1 = \btheta_2= \btheta^{(k)}$. Then by (\ref{eq:ggd}), $\btheta_1^\dag= \btheta^{(k+1)}$. The inequality in (\ref{eq:lemma-ineq}) can then be simplified as:
\begin{align*}
\begin{split}
g(\btheta^{(k+1)}) - g(\btheta^{(k)}) \le 
\alpha \bdelta (\btheta^{(k)})^\top \bm{G}_\alpha (\btheta^{(k)})
- \frac{\alpha}{2}&\norm{\bm{G}_\alpha(\btheta^{(k)})}_2^2.
\end{split}
\end{align*}
By the Cauchy-Schwarz inequality and the sufficient condition that $ \norm{\bdelta(\btheta^{(k)})}_2 < \frac{1}{2}\norm{\bm{G}_\alpha(\btheta^{(k)})}_2$, we can further simplify the inequality and conclude $g(\btheta^{(k+1)}) < g(\btheta^{(k)}) $.

\subsection{Proof of Theorem~\ref{thm:convergence-rate}}
To prove Theorem ~\ref{thm:convergence-rate}, we first review Proposition 1 in \cite{schmidt2011convergence}:
\begin{theorem}[Convergence on Average, \cite{schmidt2011convergence}]\normalfont
\label{thm:avg-convergence-rate}
Let $\mathcal{K}=(\btheta^{(0)},\btheta^{(1)},\btheta^{(2)},\cdots,\btheta^{(\kappa)})$ be the iterates generated by Algorithm~\ref{alg:pxgb}, then
\begin{align*}
\begin{split}
g\left(\frac{1}{\kappa}\sum_{k=1}^\kappa \btheta^{(k)}\right) - g(\hat{\btheta}) \le \frac{L}{2\kappa} \left(\norm{\btheta^{(0)}-\hat{\btheta}}_2 + \frac{2}{L}\sum_{k=1}^\kappa \norm{\bdelta(\btheta^{(k)})}_2 \right)^2.
\end{split}
\end{align*}
\end{theorem}
Furthermore, according to the assumption that  $g(\btheta^{(k+1)}) \le g(\btheta^{(k)})$ with $k \in \{1, 2, \cdots, \kappa \}$, we have: $g \left(\frac{1}{\kappa}\sum_{k=1}^\kappa \btheta^{(k)}\right)  \ge g(\btheta^{(\kappa)})$. Therefore,
\begin{align*}
\begin{split}
g (\btheta^{(\kappa)}) - g(\hat{\btheta}) \le \frac{L}{2\kappa} \left(\norm{\btheta^{(0)}-\hat{\btheta}}_2 + \frac{2}{L}\sum_{k=1}^\kappa \norm{\bdelta(\btheta^{(k)})}_2 \right)^2.
\end{split}
\end{align*}

\subsection{Proof of Theorem~\ref{thm:exp-err-bound}}
\renewcommand{\ss}{\Norm{\mathbb{E}_{\tilde{\bx}_\tau} [\bdelta(\btheta)\mid \tilde{\bx}_0]}_2}
\subsubsection{Proof of Lemma~\ref{lem:exp-err}}

The rationale behind our proof follow that of \cite{bengio2009justifying} and \cite{fischer2011bounding}.

Let $\tilde{\bx}_0\in\curly{0,1}^p$ be an initialization of the Gibbs sampling algorithm. Let $\btheta$ be the parameterization from which the Gibbs sampling algorithm generates new samples. A Gibbs-$\tau$ algorithm hence uses the $\tau^{th}$ sample , $\tilde{\bx}_\tau$, generated from the chain to approximate the gradient. Since there is only one Markov chain in total, we have $\mathbb{S}=\curly{\tilde{\bx}_\tau}$. The gradient approximation of Gibbs-$\tau$ is hence given by:
\begin{equation}
\label{eq:gradappx1}
\bm{\Delta} f(\btheta) = \bm{\psi}(\tilde{\bx}_\tau) - \mathbb{E}_{\mathbb{X}}\bm{\psi}(\bx).
\end{equation}
The actual gradient, $\grad f(\btheta)$, is given in (\ref{eq:grad}). Therefore, the difference between the approximation and the actual gradient is 
\begin{align*}
\bdelta(\btheta) &= \bm{\bm{\Delta}} f(\btheta) - \grad f(\btheta) =  \bm{\psi}(\tilde{\bx}_\tau) - \mathbb{E}_{\btheta}\bm{\psi}(\bx) = \grad \log \mathrm{P}_{\btheta}(\tilde{\bx}_\tau).
\end{align*}

We rewrite
\begin{equation*}
\mathrm{P}_\tau(\bx \mid \tilde{\bx}_0)  = \mathrm{P}(\tilde{\bX}_\tau = \bx \mid  \tilde{\bx}_0) = \mathrm{P}_{\btheta}(\bx)+ \epsilon_\tau(\bx),
\end{equation*}
where $\epsilon_\tau(\bx)$ is the difference between $\mathrm{P}_\tau(\bx \mid \tilde{\bx}_0)$ and $\mathrm{P}_{\btheta}(\bx)$. Consider the expectation of the $j^{th}$ component of $\bdelta(\btheta)$, $\delta_j(\btheta)$, where $j\in\curly{1,2,\cdots,m}$, after running Gibbs-$\tau$ that is initialized by $\tilde{\bx}_0$:
\begin{align}
\label{eq:exp-grad}
\begin{split}
\mathbb{E}_{\tilde{\bx}_\tau} [\delta_j(\btheta) \mid \tilde{\bx}_0]&=  \sum_{\bx\in\curly{0,1}^p}\hspace{-2mm} \mathrm{P}_\tau (\bx\mid\tilde{\bx}_0) \delta_j(\btheta)= \sum_{\bx\in\curly{0,1}^p} (\mathrm{P}_{\btheta}(\bx)+\epsilon_\tau(\bx)) \delta_j(\btheta)
\\&= \hspace{-2mm} \sum_{\bx\in\curly{0,1}^p}\hspace{-2mm} \epsilon_\tau(\bx) \delta_i(\btheta)=  \sum_{\bx\in\curly{0,1}^p} (\mathrm{P}_\tau(\bx\mid\bx_0) -\mathrm{P}_{\btheta}(\bx)) \delta_j(\btheta)
\\&= \sum_{\bx\in\curly{0,1}^p} (\mathrm{P}_\tau(\bx\mid\bx_0) - \mathrm{P}_{\btheta}(\bx)) \nabla_j \log \mathrm{P}_{\btheta}(\tilde{\bx}_\tau),
\end{split}
\end{align}
where we have used the fact that $\sum_{\bx \in\curly{0,1}^p} \mathrm{P}_{\btheta}(\bx)\nabla_j \log \mathrm{P}_{\btheta}(\bx) = 0$,
and $\nabla_j \log \mathrm{P}_{\btheta}(\bx)$ represents the $j^{th}$ component of $\grad \log \mathrm{P}_{\btheta}(\tilde{\bx}_\tau)$, with $j\in\curly{1,2,\cdots,m}$.

Therefore, from (\ref{eq:exp-grad}),
\begin{align}
\label{eq:exp-grad-abs}
\begin{split}
\lvert \mathbb{E}_{\tilde{\bx}_\tau} [\delta_j(\btheta) \mid \tilde{\bx}_0] \rvert
& \le \sum_{\bx\in\curly{0,1}^p} \lvert\mathrm{P}_\tau(\bx\mid\bx_0) -\mathrm{P}_{\btheta}(\bx)\rvert \cdot \lvert \nabla_j \log \mathrm{P}_{\btheta}(\tilde{\bx}_\tau) \rvert\\
& \le \sum_{\bx\in\curly{0,1}^p} \lvert\mathrm{P}_\tau(\bx\mid\bx_0) -\mathrm{P}_{\btheta}(\bx)\rvert  = 2\Norm{\mathrm{P}_\tau(\bx\mid\tilde{\bx}_0) -\mathrm{P}_{\btheta}(\bx)}_{\text{TV}},
\end{split}
\end{align}
where we have used the fact that $\lvert \nabla_j \log \mathrm{P}_{\btheta}(\tilde{\bx}_\tau) \rvert \le 1$ when $\bm{\psi}(\bx) \in \curly{0,1}^m$, for all $\bx \in \curly{0,1}^p$.

Therefore, by (\ref{eq:exp-grad-abs}),
\begin{align*}
\lVert \mathbb{E}_{\tilde{\bx}_\tau} [\bdelta(\btheta) \mid \tilde{\bx}_0]\rVert_2 = & \sqrt{
\sum_{j=1}^m \lvert \mathbb{E}_{\tilde{\bx}_\tau} [\delta_j(\btheta) \mid \tilde{\bx}_0] \rvert^2}
\le  \sqrt{m \times (2\Norm{\mathrm{P}_\tau(\bx\mid\bx_0) -\mathrm{P}_{\btheta}(\bx)}_{\text{TV}})^2}
\\= &2\sqrt{m}\Norm{\mathrm{P}_\tau(\bx\mid\bx_0) -\mathrm{P}_{\btheta}(\bx)}_{\text{TV}}.
\end{align*}

\subsubsection{Proof of Lemma~\ref{lem:u-bpmn}}

Let $j\ne i$ be given. With $\xi_{ij} = \theta_{\min\curly{i,j},\max\curly{i,j}}$, consider
\begin{align*}
\mathrm{P}_{\btheta}(X_i=1\mid\bX_{-i}) = & \frac{\mathrm{P}_{\btheta}(X_i=1 , \bX_{-i})}{\mathrm{P}_{\btheta}(X_i=0,  \bX_{-i})+\mathrm{P}_{\btheta}(X_i=1 , \bX_{-i})}
\\=& \frac{1}{1+\exp \left( -\theta_{ii}-\sum_{k\ne i}\xi_{i, k} X_k \right)}
\\= & \frac{1}{1+\exp \left(-\theta_{ii}-\sum_{k\ne i,k\ne j}\xi_{i, k}X_k \right) \exp\left(-\xi_{i, j}X_j\right)}\\
= & g\left(\exp\left(-\xi_{i, j}X_j\right), b_1\right),
\end{align*}
where
\begin{equation*}
b = \exp \left(-\theta_{ii}-\sum_{k\ne i,k\ne j}\xi_{i, k}X_k \right) \in [r,s],
\end{equation*}
with\small
\begin{gather*}
r = \exp \left(- \theta_{ii} - \sum_{k\ne i,k\ne j} \xi_{i, k} \max\curly{\mathrm{sgn}(\xi_{i, k}),0}   \right),\quad s = \exp \left( - \theta_{ii} -\sum_{k\ne i,k\ne j}\xi_{i, k} \max\curly{-\mathrm{sgn}(\xi_{i, k}),0}  \right).
\end{gather*}\normalsize

Therefore,
\begin{align*}
C_{ij} =& \max_{\bX,\bY\in N_j}  \frac{1}{2} \lvert \mathrm{P}_{\btheta}(X_i=1\mid\bX_{-i}) - \mathrm{P}_{\btheta}(Y_i=1\mid\bY_{-i})\rvert + \frac{1}{2} \lvert \mathrm{P}_{\btheta}(X_i=0\mid\bX_{-i}) - \mathrm{P}_{\btheta}(Y_i=0\mid\bY_{-i})\rvert\\
=&\max_{\bX,\bY\in N_j}   \lvert \mathrm{P}_{\btheta}(X_i=1\mid\bX_{-i}) - \mathrm{P}_{\btheta}(Y_i=1\mid\bY_{-i})\rvert
\\ =& \max_{\bX,\bY\in N_j}\lvert g\left(\exp\left(-\xi_{i, j}X_j\right), b\right) - g\left(\exp\left(-\xi_{i, j}Y_j\right), b\right)\rvert \\
=& \max_{\bX,\bY\in N_j} \frac{\lvert \exp\left(-\xi_{i, j}X_j\right) - \exp\left(-\xi_{i, j}Y_j\right) \rvert b}{\left(1+b\exp\left(-\xi_{i, j}X_j\right)\right)\left(1+b_1\exp\left(-\xi_{i, j}Y_j\right)\right)}
\\=& \max_{\bX,\bY\in N_j}\frac{\lvert \exp\left(-\xi_{i, j}\right) - 1 \rvert b}{\left(1+b\exp\left(-\xi_{i, j}\right)\right)(1+b)}.
\end{align*}
Then following the Lemma 15 in \cite{mitliagkas2017improving}, we have
\begin{equation}
\begin{gathered}
C_{ij}\le \frac{\lvert \exp\left(-\xi_{i, j}\right) - 1 \rvert b^*}{\left(1+b_1^*\exp\left(-\xi_{i, j}\right)\right)(1+b^*)},
\end{gathered}
\end{equation}
with $b^* = \max\curly{r,\min\curly{s,\exp\left(\frac{\xi_{i, j}}{2}\right)}}$.

\subsection{Proof of Theorem~\ref{thm:sample-err-bound}}
\label{sec:proof-non-asy-bound}
We are interested in concentrating $\Norm{\bdelta(\btheta)}_2$ around $\Norm{\mathbb{E}_{\tilde{\bx}_\tau} [\bdelta(\btheta) \mid \tilde{\bx}_0]}_2$. To this end, we first consider concentrating $\delta_j(\btheta)$ around  $\mathbb{E}_{\tilde{\bx}_\tau} [\delta_j(\btheta) \mid \tilde{\bx}_0]$, where $j\in\curly{1,2,\cdots,m}$. Let $q$ defined in Algorithm~\ref{alg:grad} be given. Then $q$ trials of Gibbs sampling are run, resulting in $\curly{\delta_j^{(1)}(\btheta) , \delta_j^{(2)}(\btheta) ,\cdots, \delta_j^{(q)}(\btheta)}$, and $\curly{\psi_j^{(1)}(\btheta) , \psi_j^{(2)}(\btheta) ,\cdots, \psi_j^{(q)}(\btheta)}$ defined in Section~\ref{sec:non-asy-bound}, one element for each of the $q$ trials. Since all the trials are independent, $ \delta_j^{(i)}(\btheta)$'s can be considered as i.i.d.~samples with mean $\mathbb{E}_{\tilde{\bx}_\tau} [\delta_j(\btheta) \mid \tilde{\bx}_0]$. Furthermore, $\delta_j^{(i)}(\btheta) = \nabla_j \log \mathrm{P}_{\btheta}(\tilde{\bx}_\tau) \in [-1,1]$ when $\bm{\psi}(\bx) \in \curly{0,1}^m$, for all $\bx \in \curly{0,1}^p$. Let $\beta_j>0$ be given; we define the adversarial event:
\begin{equation}
\begin{gathered}
\label{eq:ad-event}
E_j^q(\epsilon_j) = \left\lvert \frac{1}{q}\sum_{i=1}^q \delta_j^{(i)}(\btheta) -  \mathbb{E}_{\tilde{\bx}_\tau} [\delta_j(\btheta) \mid \tilde{\bx}_0]  \right\rvert >\epsilon_j, 
\end{gathered}
\end{equation}
with $j\in\curly{1,2,\cdots,m}$.

Define another random variable $Z_j = \frac{1 + \delta_j(\btheta)}{2}$ with samples $Z_j^{(i)} = \frac{1 + \delta^{(i)}_j(\btheta)}{2}$ and the sample variance $V_{Z_j} = \frac{V_{\delta_j}}{4} = \frac{V_{\psi_j}}{4}$.

Considering $Z \in [0, 1]$, we can apply Theorem 4 in \cite{maurer2009empirical} and achieve 
\begin{equation*}
\mathrm{P}\left( \left\lvert \frac{1}{q}\sum_{i=1}^q Z_j^{(i)} -  \mathbb{E}_{\tilde{\bx}_\tau} [Z_j \mid \tilde{\bx}_0]  \right\rvert >\frac{\epsilon_j}{2} \right) \le 2 \beta_j,
\end{equation*}
where 
\begin{align*}
\frac{\epsilon_j}{2} &= \sqrt{\frac{2V_{Z_j}\ln 2/\beta_j } {q}} + \frac{7\ln2/\beta_j}{3(p-1)} =\sqrt{\frac{V_{\psi_j}\ln 2/\beta_j } {2q}} + \frac{7\ln2/\beta_j}{3(p-1)}.
\end{align*}
That is to say
\begin{equation*}
\mathrm{P}\left(E_j^q(\epsilon_j) \right) \le 2 \beta_j.
\end{equation*}

Now, for all $j\in\curly{1,2,\cdots,m}$, we would like $\frac{1}{m}\sum_{i=1}^m \delta_j^{(i)}(\btheta)$ to be close to $\mathbb{E}_{\tilde{\bx}_\tau} [\delta_j(\btheta) \mid \tilde{\bx}_0]$. i.e., 
\begin{equation*}
\left\lvert \frac{1}{q}\sum_{i=1}^q \delta_j^{(i)}(\btheta) - \mathbb{E}_{\tilde{\bx}_\tau} [\delta_j(\btheta) \mid \tilde{\bx}_0] \right\rvert \le \epsilon_j.
\end{equation*}
This concentrated event will occur with probability:
\begin{align*}
1 & - \mathrm{P}\left(E_j(\epsilon_j) \right)\ge 1- \mathrm{P}\left(E_j^q(\epsilon_j) \right)\ge 1 - 2 \beta_j.
\end{align*}
When all the concentrated events occur for each $j$,
\begin{align*}
 \Norm{ \bdelta(\btheta)}_2 - \Norm{\mathbb{E}_{\tilde{\bx}_\tau} [\bdelta(\btheta) \mid \tilde{\bx}_0]}_2 \le&  \left\lVert \bdelta(\btheta) - \mathbb{E}_{\tilde{\bx}_\tau} [\bdelta(\btheta) \mid \tilde{\bx}_0] \right\rVert_2=  \left\lVert \frac{1}{q}\sum_{i=1}^q \bdelta^{(i)}(\btheta) - \mathbb{E}_{\tilde{\bx}_\tau} [\bdelta(\btheta) \mid \tilde{\bx}_0] \right\rVert_2
\\= & \sqrt{\sum_{j=1}^m \left(\frac{1}{q}\sum_{i=1}^q \delta_j^{(i)}(\btheta) -  \mathbb{E}_{\tilde{\bx}_\tau} [\delta_j(\btheta) \mid \tilde{\bx}_0]\right)^2} \le  \sqrt{\sum_{ j = 1}^m \epsilon_j^2}.
\end{align*}
Therefore,
\begin{align*}
\Norm{ \bdelta(\btheta)}_2 & \le \Norm{\mathbb{E}_{\tilde{\bx}_\tau} [\bdelta(\btheta) \mid \tilde{\bx}_0]}_2 + \sqrt{\sum_{ j = 1}^m \epsilon_j^2}\le 2\sqrt{m}\Norm{\mathrm{P}_\tau(\bx\mid\tilde{\bx}_0) -\mathrm{P}_{\btheta}(\bx)}_{\text{TV}}+\sqrt{\sum_{ j = 1}^m \epsilon_j^2}
\\& \le  2\sqrt{m}\left(\mathscr{G}(\mathbf{B}^\tau)+\sqrt{\frac{\sum_{ j = 1}^m \epsilon_j^2}{4m}}\right).
\end{align*}

That is to say, we can conclude that (\ref{eq:sample-err-bound}) holds provided that all the concentrated events occur. Thus, the probability that (\ref{eq:sample-err-bound}) holds follows the inequality below:
\begin{align*}
\mathrm{P}\left(\Norm{ \bdelta(\btheta)}_2 \le 2\sqrt{m}\left(\mathscr{G}(\mathbf{B}^\tau)+\sqrt{\frac{\sum_{ j = 1}^m \epsilon_j^2}{4m}}\right)\right) \ge &1 - \mathrm{P}\left(\bigcup_{j=1}^m E_j(\epsilon_j) \right)\ge 1-\sum_{j=1}^m \mathrm{P}\left(E_j^q(\epsilon_j) \right)
\ge 1 - 2 \sum_{j = 1}^m \beta_j.
\end{align*}

\subsection{Proof of Theorem~\ref{thm:criterion-decrease}}
\label{sec:proof-criterion-decrease}
We consider the probability that the achieved objective function value decreases in the $k^{th}$ iteration provided that the criterion \ref{eq:criterion-aggressive} is satisfied:
\begin{align*}
\mathrm{P}\left(g(\btheta^{(k+1)}) < g(\btheta^{(k)}) \given 2\sqrt{m} \mathscr{G}(\mathbf{B}^\tau) < \frac{1}{2}\norm{\bm{G}_\alpha(\btheta^{(k)})}_2 \right).
\end{align*}
Since  $\norm{\bdelta(\btheta^{(k)})}_2 \le \frac{1}{2}\norm{\bm{G}_\alpha(\btheta^{(k)})}_2$ provided in Theorem~\ref{thm:obj-decrease} is a sufficient condition for $g(\btheta^{(k+1)}) \le g(\btheta^{(k)})$, we have:
\begin{align*}
&\mathrm{P}\left(g(\btheta^{(k+1)}) \le g(\btheta^{(k)}) \given  2\sqrt{m} \mathscr{G}(\mathbf{B}^\tau)< \frac{1}{2}\norm{\bm{G}_\alpha(\btheta^{(k)})}_2 \right) \\ 
\ge &\mathrm{P}\left(\norm{\bdelta(\btheta^{(k)})}_2 \le \frac{1}{2}\norm{\bm{G}_\alpha(\btheta^{(k)})}_2 \given 2\sqrt{m} \mathscr{G}(\mathbf{B}^\tau) < \frac{1}{2}\norm{\bm{G}_\alpha(\btheta^{(k)})}_2 \right)\\
= & 1 - \mathrm{P}\left(\norm{\bdelta(\btheta^{(k)})}_2 > \frac{1}{2}\norm{\bm{G}_\alpha(\btheta^{(k)})}_2\given 2\sqrt{m} \mathscr{G}(\mathbf{B}^\tau) < \frac{1}{2}\norm{\bm{G}_\alpha(\btheta^{(k)})}_2 \right)\\
\ge & 1 - \mathrm{P}\left(\norm{\bdelta(\btheta^{(k)})}_2 - \Norm{\mathbb{E}_{\tilde{\bx}_\tau} [\bdelta(\btheta) \mid \tilde{\bx}_0]}_2> \frac{1}{2}\norm{\bm{G}_\alpha(\btheta^{(k)})}_2-2\sqrt{m} \mathscr{G}(\mathbf{B}^\tau) \given 2\sqrt{m} \mathscr{G}(\mathbf{B}^\tau) < \frac{1}{2}\norm{\bm{G}_\alpha(\btheta^{(k)})}_2 \right)\\
\ge & 1-\sum_{j=1}^m\mathrm{P}\left(E_j^q(\frac{1}{2\sqrt{m}}\norm{\bm{G}_\alpha(\btheta^{(k)})}_2-2 \mathscr{G}(\mathbf{B}^\tau))\given  2\sqrt{m} \mathscr{G}(\mathbf{B}^\tau)< \frac{1}{2}\norm{\bm{G}_\alpha(\btheta^{(k)})}_2 \right),
\end{align*}
where $E_j^q \left(\frac{1}{2\sqrt{m}}\norm{\bm{G}_\alpha(\btheta^{(k)})}_2-2 \mathscr{G}(\mathbf{B}^\tau)\right)$ is defined in (\ref{eq:ad-event}) and in the $4^{th}$ line we apply (\ref{eq:exp-err-bound}). 
As $q$ approaches infinity, by the weak law of large numbers, we have
\begin{equation*}
\lim_{q \to \infty}\mathrm{P}\left(E_j^q\left((\frac{1}{2\sqrt{m}}\norm{\bm{G}_\alpha(\btheta^{(k)})}_2-2 \mathscr{G}(\mathbf{B}^\tau)\right)\right) = 0.
\end{equation*}
Then,
\begin{align*}
&\lim_{q \to \infty}\mathrm{P}\left(g(\btheta^{(k+1)}) < g(\btheta^{(k)}) \given  2\sqrt{m} \mathscr{G}(\mathbf{B}^\tau) < \frac{1}{2}\norm{\bm{G}_\alpha(\btheta^{(k)})}_2 \right)\\
\ge& 1-\lim_{q \to \infty}\sum_{j=1}^m\mathrm{P}\left(E_j^q(\frac{1}{2\sqrt{m}}\norm{\bm{G}_\alpha(\btheta^{(k)})}_2-2 \mathscr{G}(\mathbf{B}^\tau))\given 2\sqrt{m} \mathscr{G}(\mathbf{B}^\tau)< \frac{1}{2}\norm{\bm{G}_\alpha(\btheta^{(k)})}_2 \right)
=1.
\end{align*}
\subsection{Proof of Theorem~\ref{thm:convergence-rate-tay}}
According to Theorem~\ref{thm:convergence-rate}, we only need to show 
\begin{equation*}
\lim_{q \to \infty}\mathrm{P}\left(g(\btheta^{(k+1)}) \le g(\btheta^{(k)}) \right) = 1,
\end{equation*}
for $ k = 1,2,\cdots, \kappa -1$.

By a union bound, the following inequality is true:
\begin{align*}
\begin{split}
&\lim_{q \to \infty}\mathrm{P}\left(g(\btheta^{(k+1)}) \le g(\btheta^{(k)}) \right) \leq 1 - \sum_{k=1}^{\kappa-1} \lim_{q \to \infty}\mathrm{P}\left(g(\btheta^{(k+1)}) > g(\btheta^{(k)}) \right).
\end{split}
\end{align*}

Notice that, following TAY, we always have:
\begin{equation*}
\mathrm{P}\left(2\sqrt{m} \mathscr{G}(\mathbf{B}^\tau)< \frac{1}{2}\norm{\bm{G}_\alpha(\btheta^{(k)})}_2\right) = 1,
\end{equation*}
suggesting 
\begin{align*}
\begin{split}
\lim_{q \to \infty}\mathrm{P}\left(g(\btheta^{(k+1)}) > g(\btheta^{(k)}) \right) =\lim_{q \to \infty}\mathrm{P}\left(g(\btheta^{(k+1)}) > g(\btheta^{(k)}) \given 2\sqrt{m} \mathscr{G}(\mathbf{B}^\tau) < \frac{1}{2}\norm{\bm{G}_\alpha(\btheta^{(k)})}_2   \right) 
= 0,
\end{split}
\end{align*}
where the equality is due to Theorem~\ref{thm:criterion-decrease}.

Finally, with Theorem~\ref{thm:convergence-rate}, we can finish the proof.

\section{Experiments}
\subsection{Comparison with SPG-based Methods}
In this section, we consider the effect of the regularization parameter $\lambda$. Specifically, we apply the methods mentioned the Section~\ref{sec:structure-learning} with different $\lambda$s. The results are reported in Figure~\ref{fig:str1} and Figure~\ref{fig:str2}.

\begin{figure*}[h]
\centering
\begin{subfigure}{0.32\textwidth}
\centering
\includegraphics[scale=0.22]{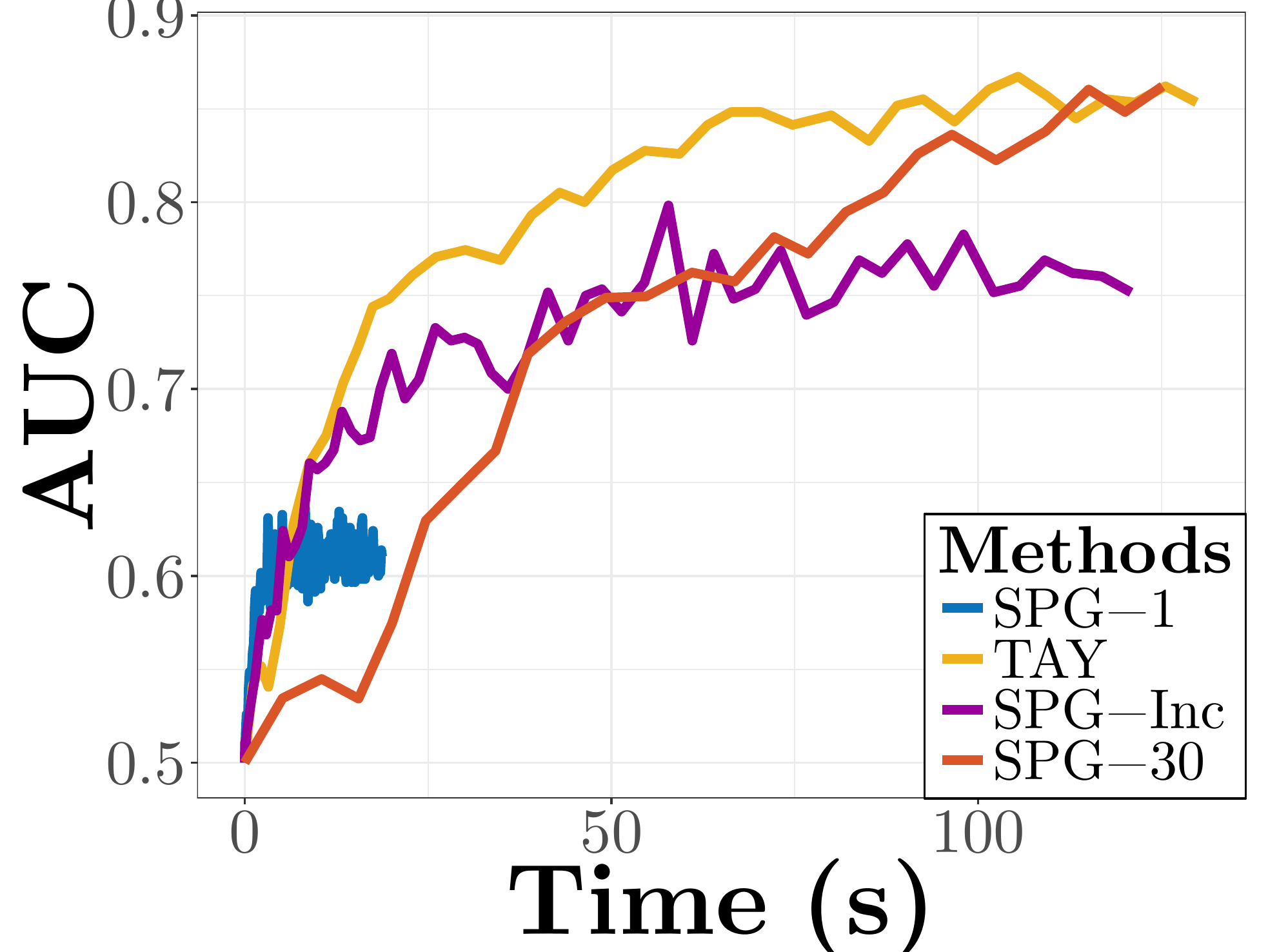}
\caption{AUC v.s. Time \\$\lambda = 0.02$}
\end{subfigure}
\centering
\begin{subfigure}{0.32\textwidth}
\centering
\includegraphics[scale=0.22]{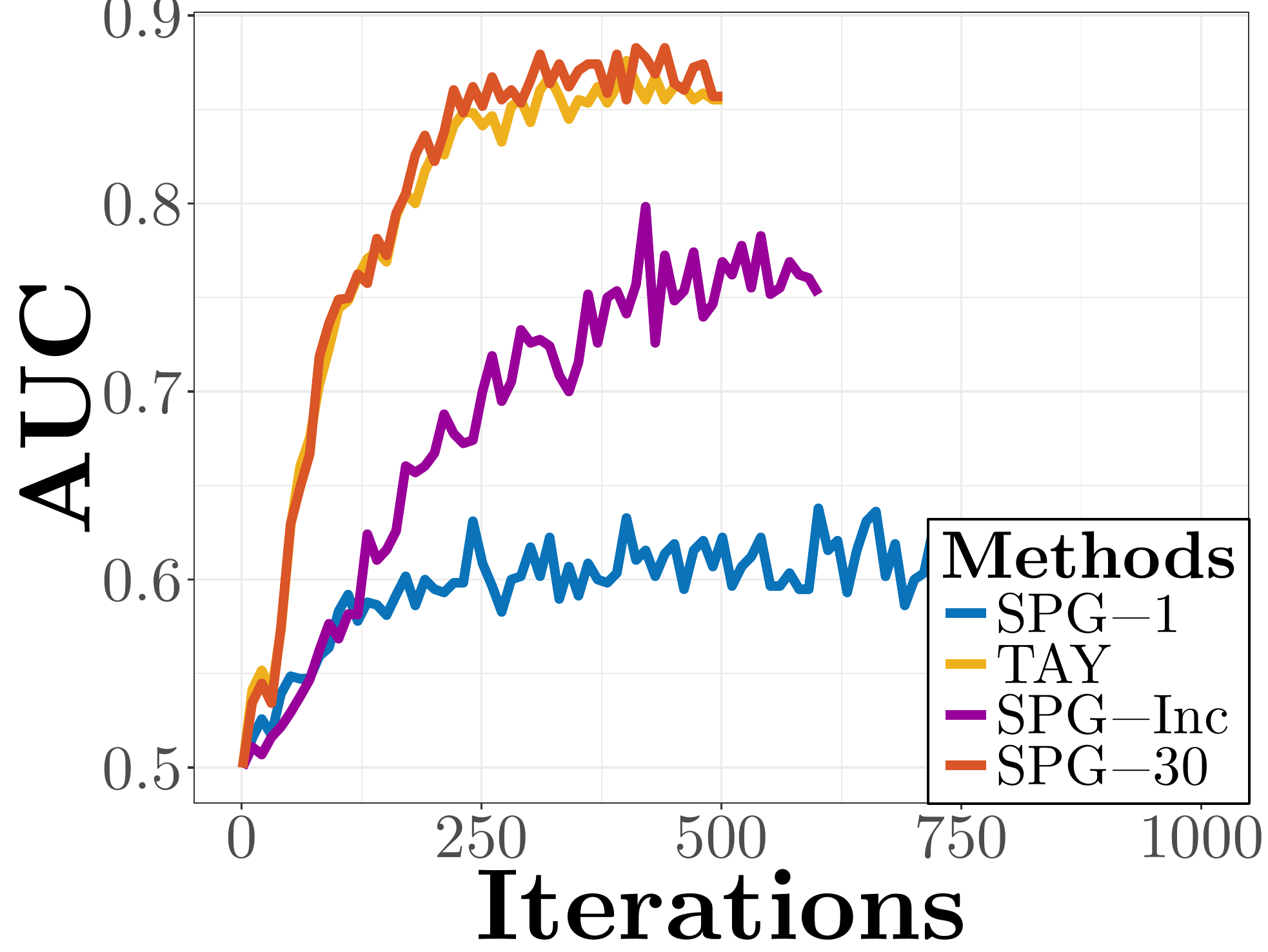}
\caption{AUC v.s. Iterations\\$\lambda = 0.02$}
\end{subfigure}
\centering
\begin{subfigure}{0.32\textwidth}
\centering
\includegraphics[scale=0.22]{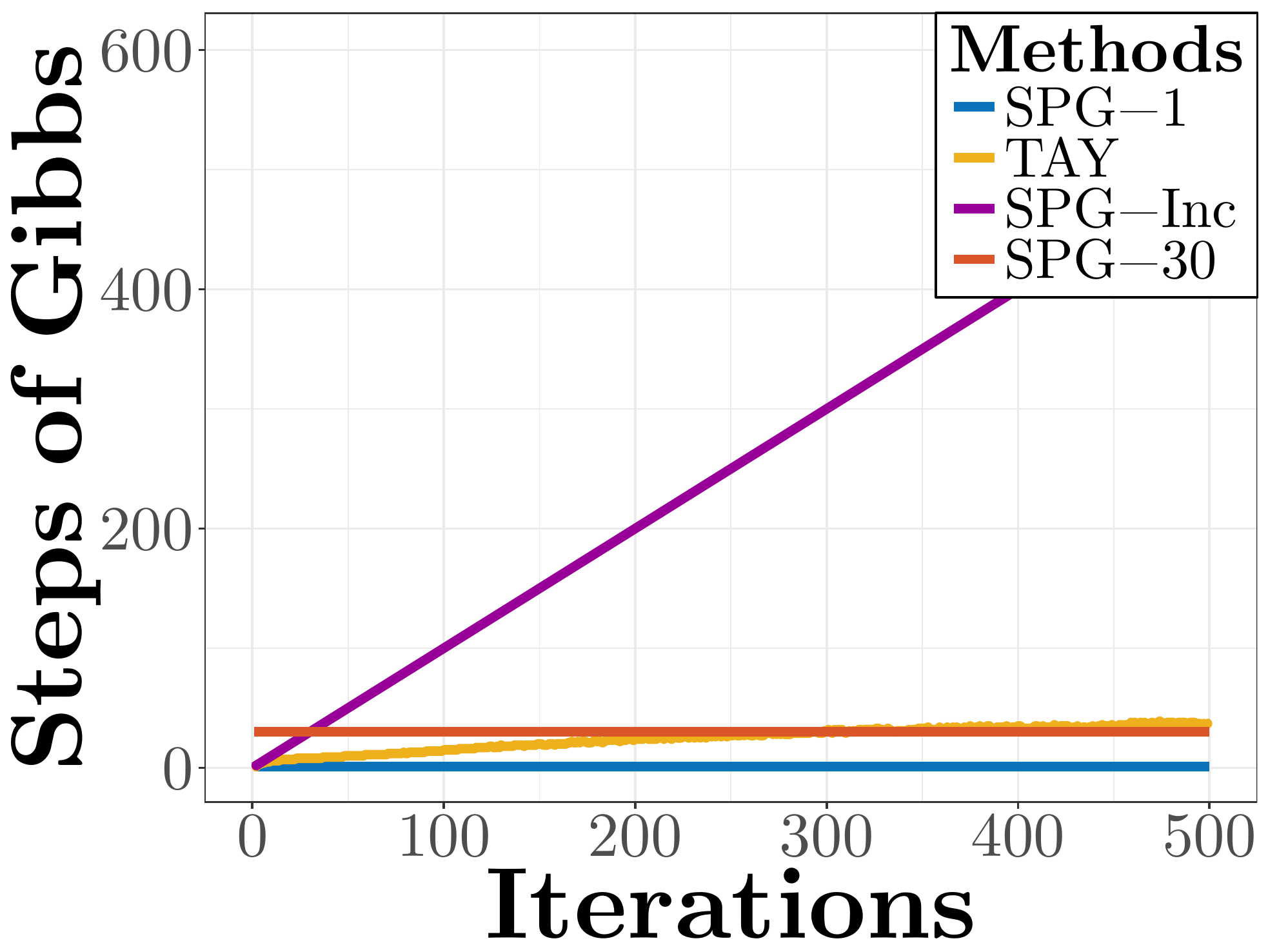}
\caption{$\tau$ v.s. Iterations\\$\lambda = 0.02$}
\end{subfigure}
\centering
\begin{subfigure}{0.32\textwidth}
\centering
\includegraphics[scale=0.22]{./support/structureLearningAUC1025.pdf}
\caption{AUC v.s. Time\\$\lambda = 0.025$}
\end{subfigure}
\begin{subfigure}{0.32\textwidth}
\centering
\includegraphics[scale=0.22]{./support/structureLearningAUCStep1025.pdf}
\caption{AUC v.s. Iterations\\$\lambda = 0.025$}
\end{subfigure}
\centering
\begin{subfigure}{0.32\textwidth}
\centering
\includegraphics[scale=0.22]{./support/structureLearningTau1025.pdf}
\caption{$\tau$ v.s. Iterations\\$\lambda = 0.025$}
\end{subfigure}
\centering
\begin{subfigure}{0.32\textwidth}
\centering
\includegraphics[scale=0.22]{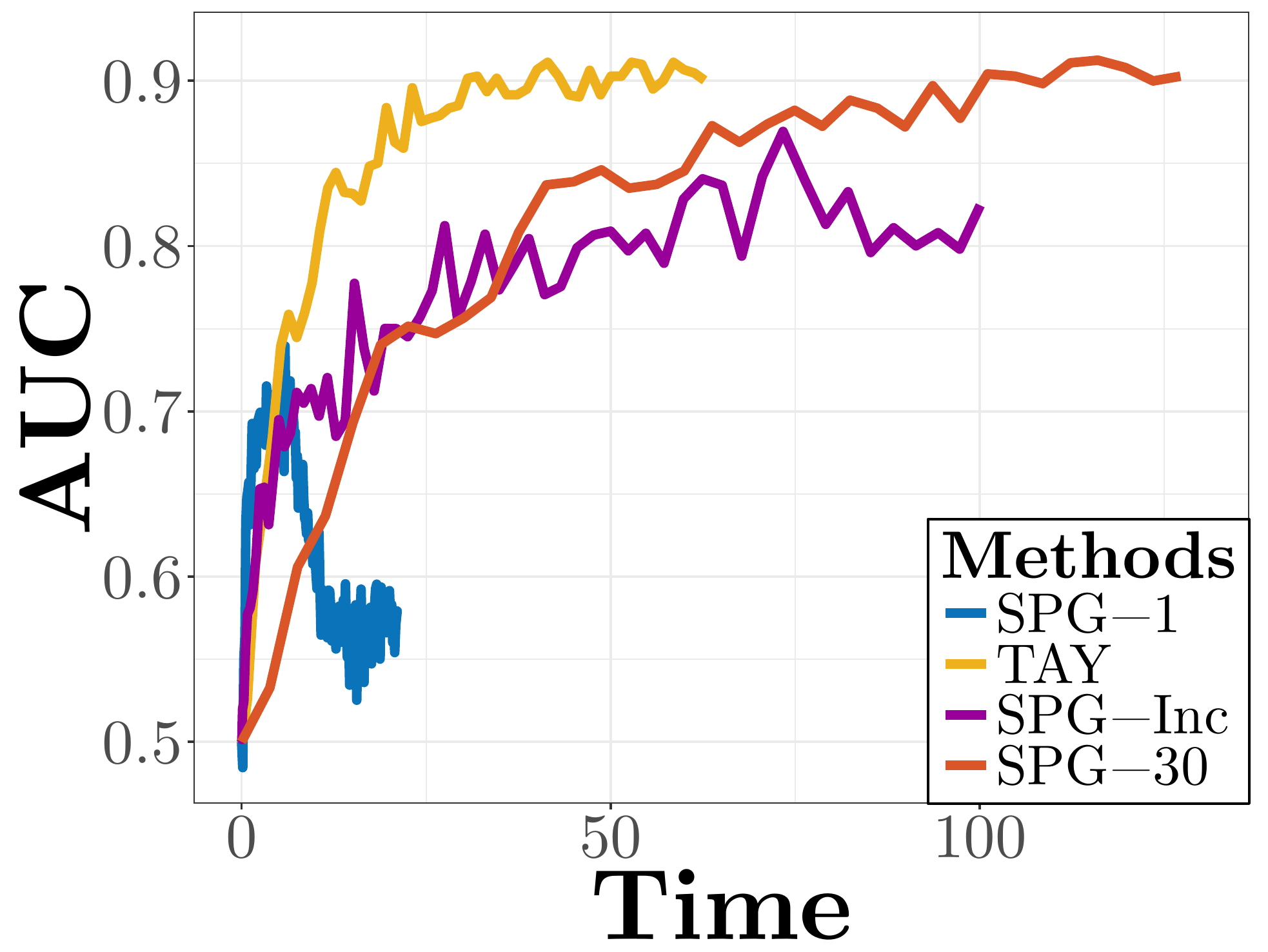}
\caption{AUC v.s. Time \\$\lambda = 0.03$}
\end{subfigure}
\centering
\begin{subfigure}{0.32\textwidth}
\centering
\includegraphics[scale=0.22]{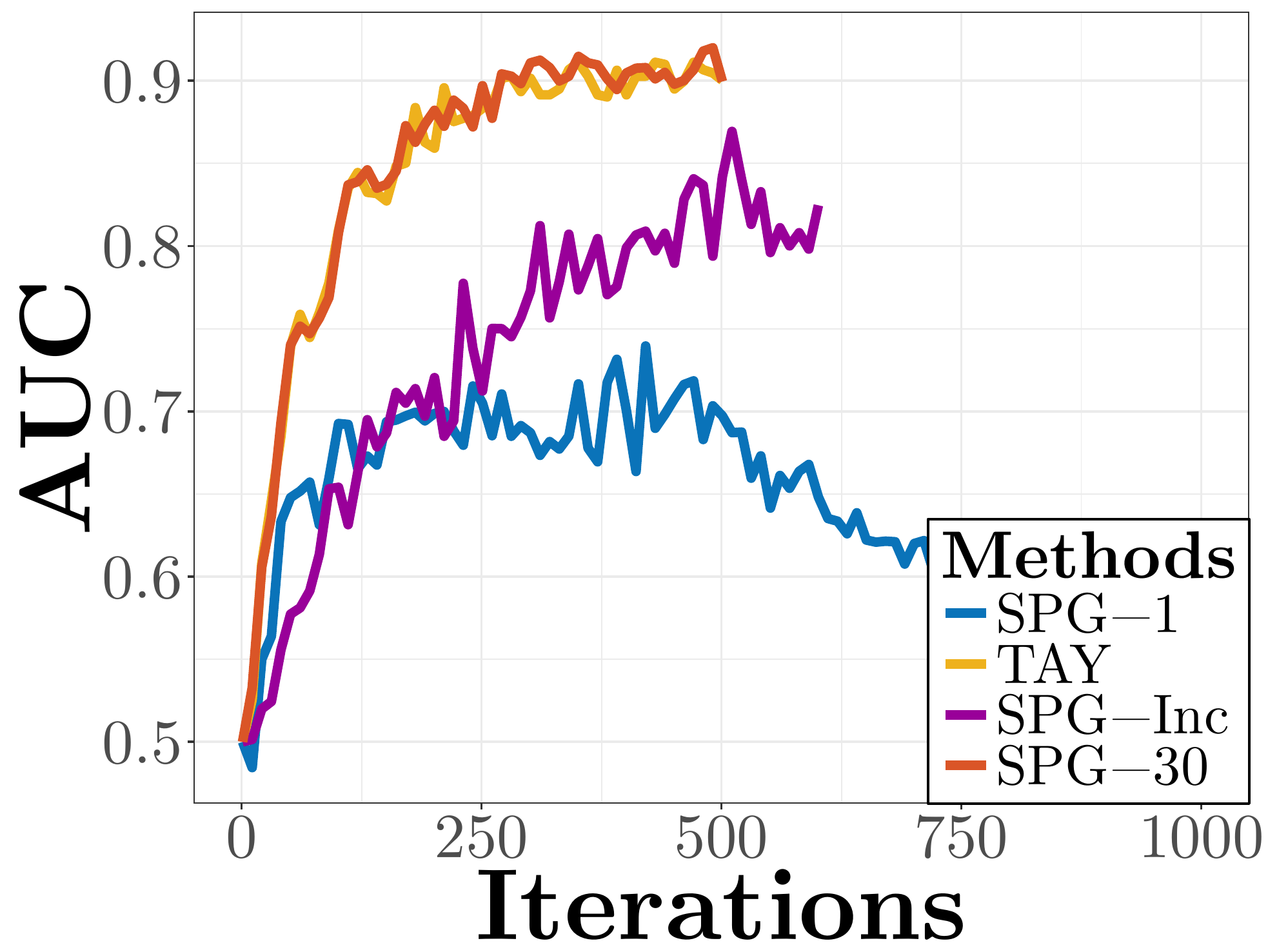}
\caption{AUC v.s. Iterations\\$\lambda = 0.03$}
\end{subfigure}
\centering
\begin{subfigure}{0.32\textwidth}
\centering
\includegraphics[scale=0.22]{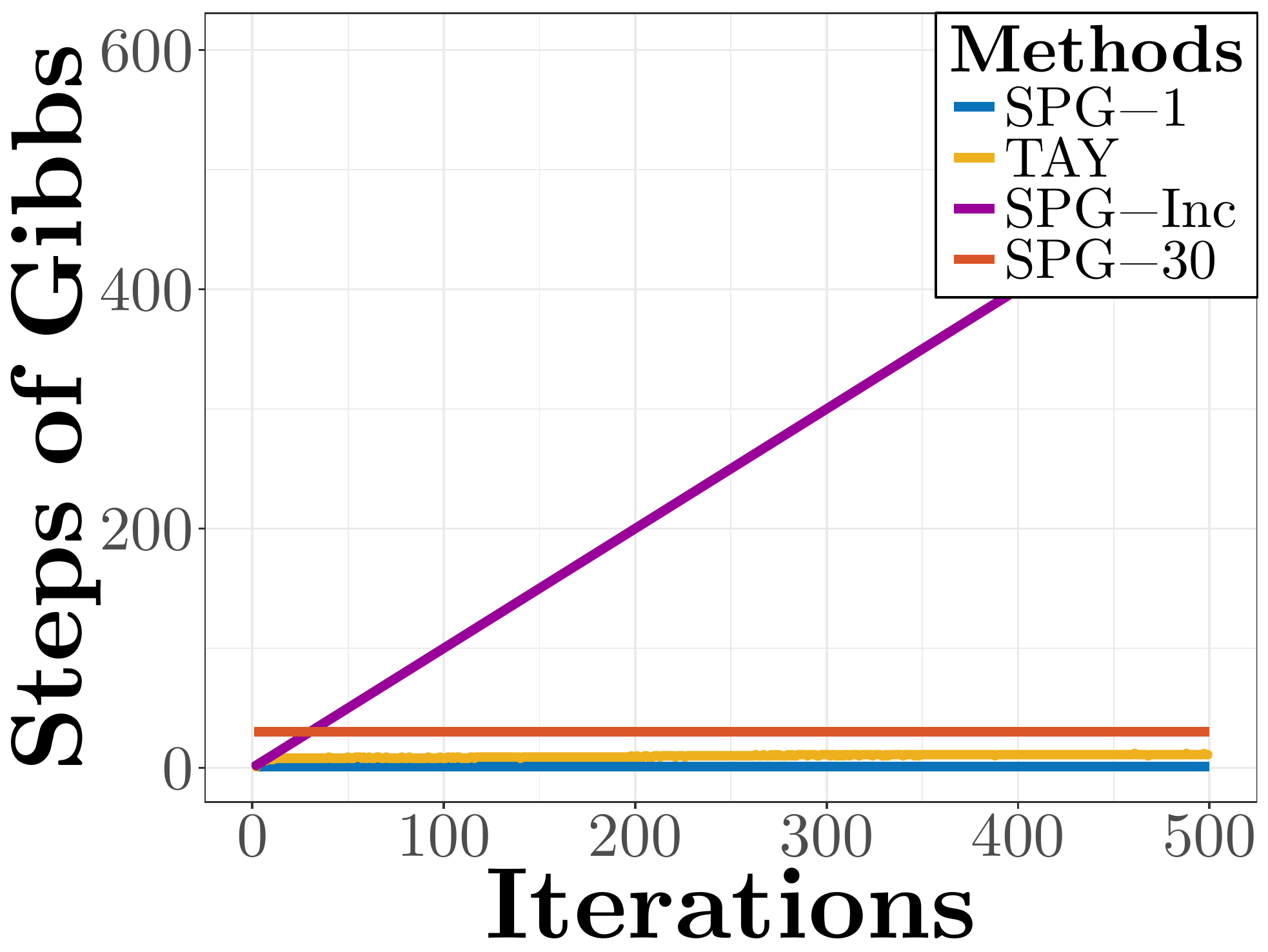}
\caption{$\tau$ v.s. Iterations\\$\lambda = 0.03$}
\end{subfigure}
\caption{Area under curve (AUC) and the steps of Gibbs sampling ($\tau$) for the structure learning of a 10-node network with different $\lambda$'s.}
\label{fig:str1}
\end{figure*}

\begin{figure*}[h]
\centering
\begin{subfigure}{0.32\textwidth}
\centering
\includegraphics[scale=0.22]{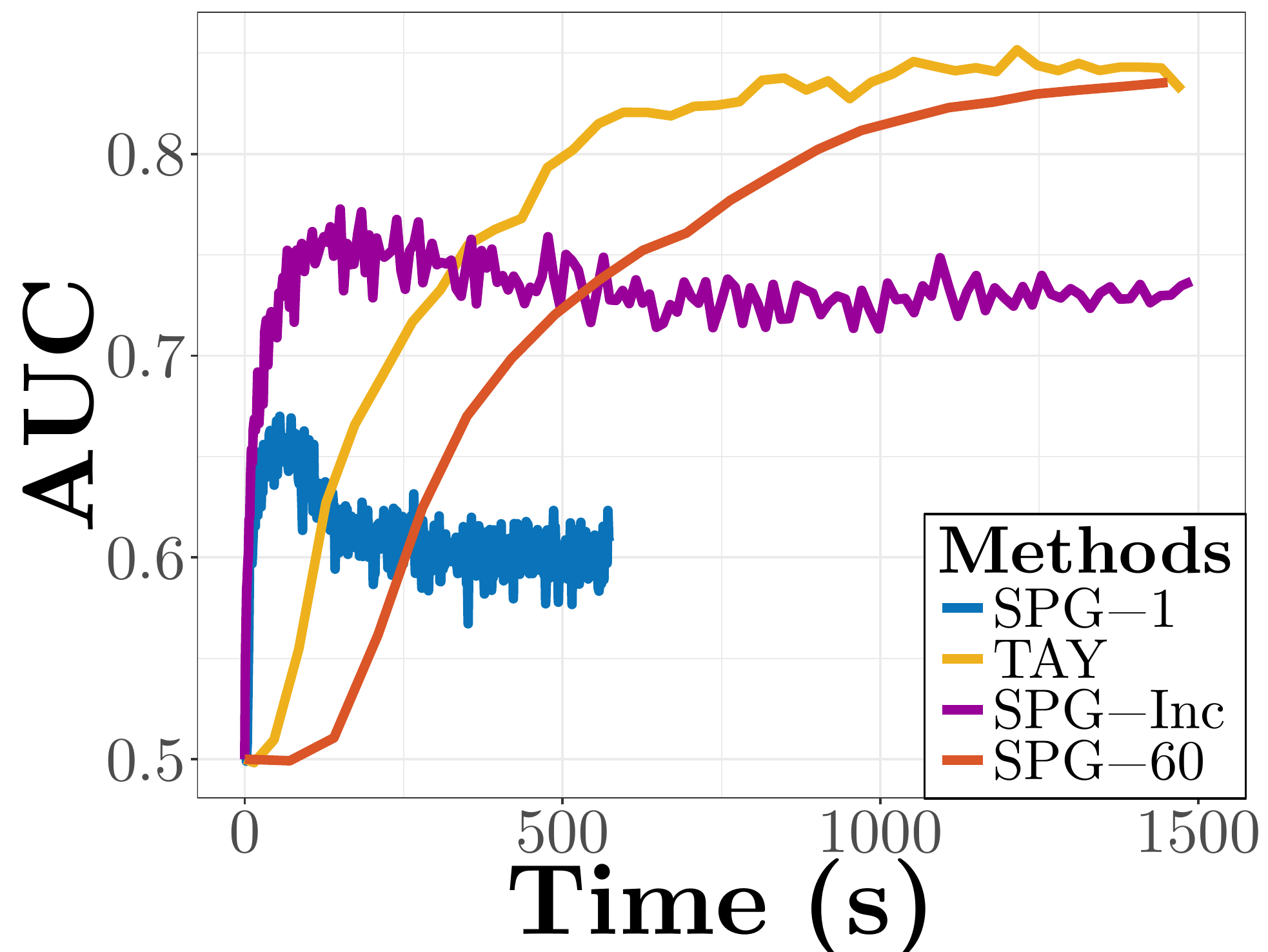}
\caption{AUC v.s. Time \\$\lambda = 0.015$}
\end{subfigure}
\centering
\begin{subfigure}{0.32\textwidth}
\centering
\includegraphics[scale=0.22]{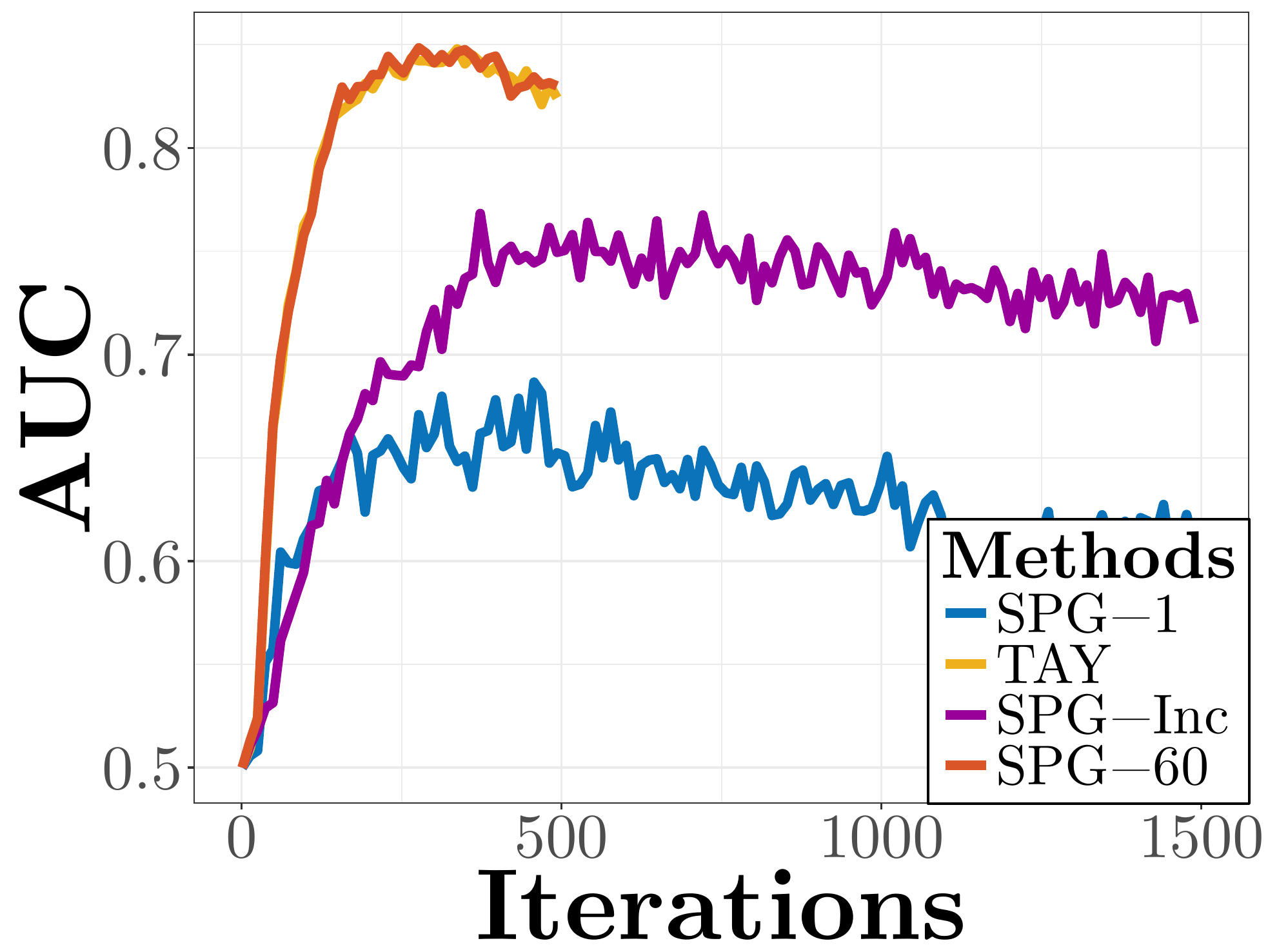}
\caption{AUC v.s. Iterations\\$\lambda = 0.015$}
\end{subfigure}
\centering
\begin{subfigure}{0.32\textwidth}
\centering
\includegraphics[scale=0.22]{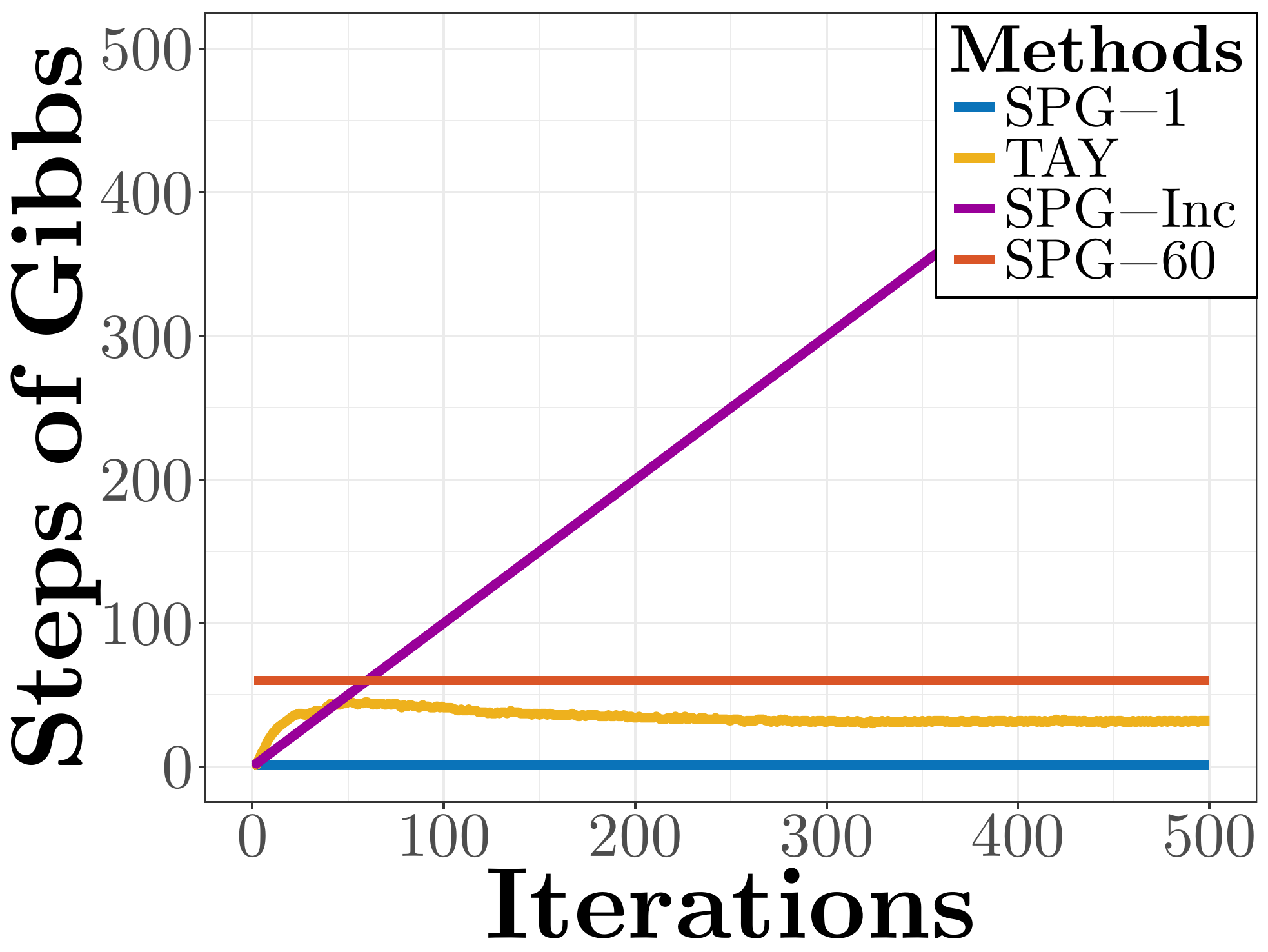}
\caption{$\tau$ v.s. Iterations\\$\lambda = 0.015$}
\end{subfigure}
\centering
\begin{subfigure}{0.32\textwidth}
\centering
\includegraphics[scale=0.22]{./support/structureLearningAUC20.pdf}
\caption{AUC v.s. Time\\$\lambda = 0.017$}
\end{subfigure}
\begin{subfigure}{0.32\textwidth}
\centering
\includegraphics[scale=0.22]{./support/structureLearningAUCStep20.pdf}
\caption{AUC v.s. Iterations\\$\lambda = 0.017$}
\end{subfigure}
\centering
\begin{subfigure}{0.32\textwidth}
\centering
\includegraphics[scale=0.22]{./support/structureLearningTau20.pdf}
\caption{$\tau$ v.s. Iterations\\$\lambda = 0.017$}
\end{subfigure}
\centering
\begin{subfigure}{0.32\textwidth}
\centering
\includegraphics[scale=0.22]{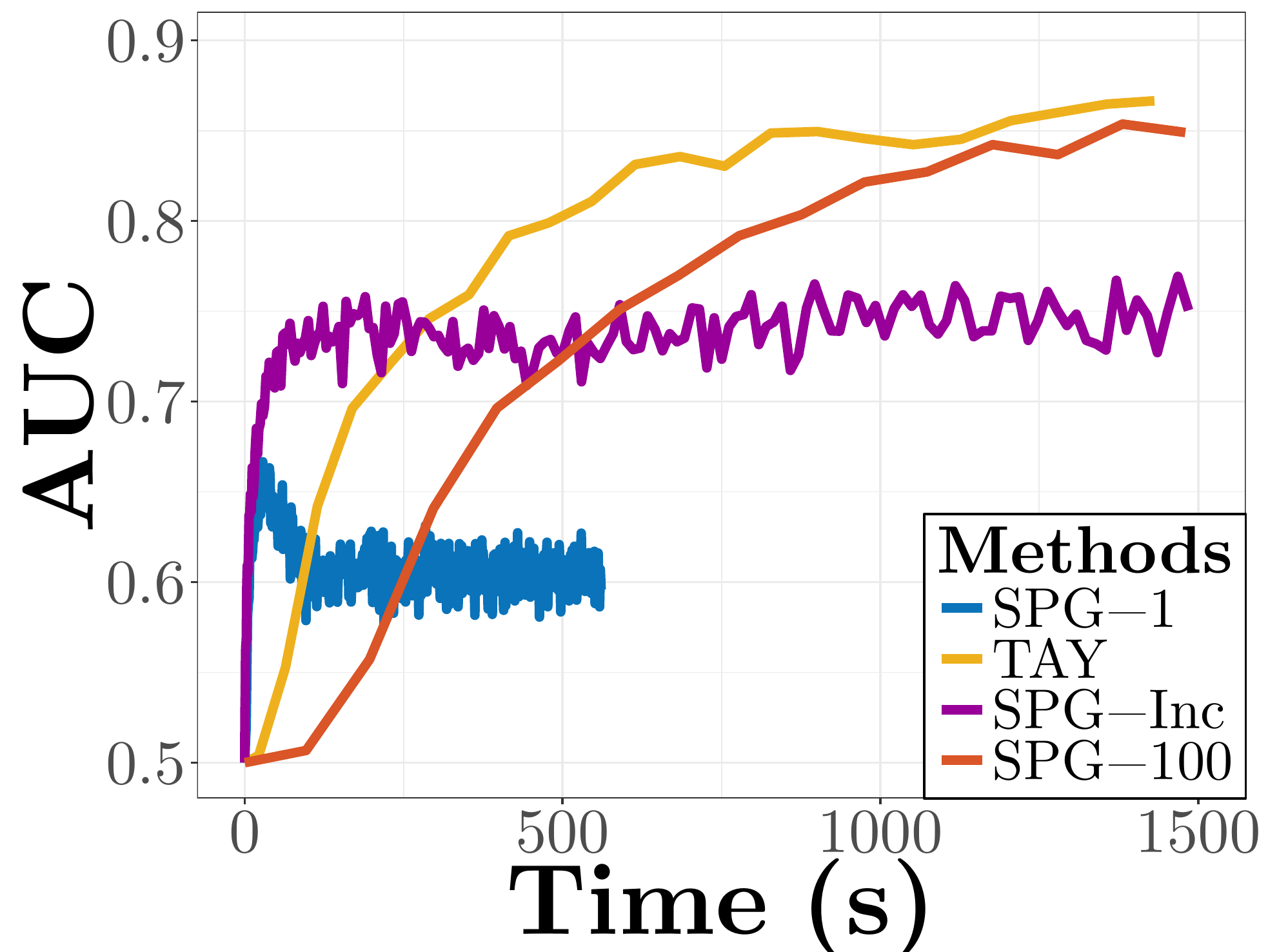}
\caption{AUC v.s. Time \\$\lambda = 0.02$}
\end{subfigure}
\centering
\begin{subfigure}{0.32\textwidth}
\centering
\includegraphics[scale=0.22]{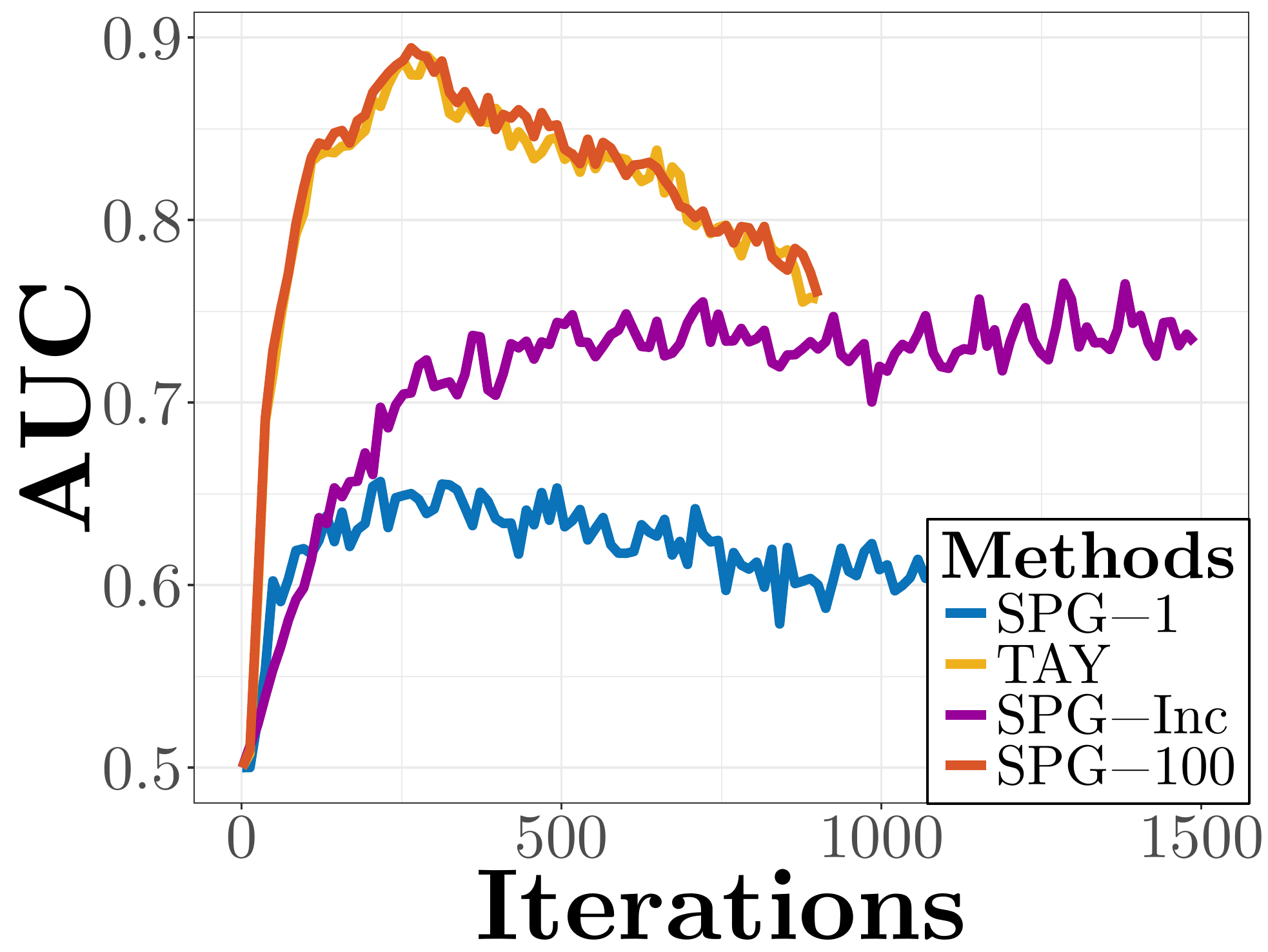}
\caption{AUC v.s. Iterations\\$\lambda = 0.02$}
\end{subfigure}
\centering
\begin{subfigure}{0.32\textwidth}
\centering
\includegraphics[scale=0.22]{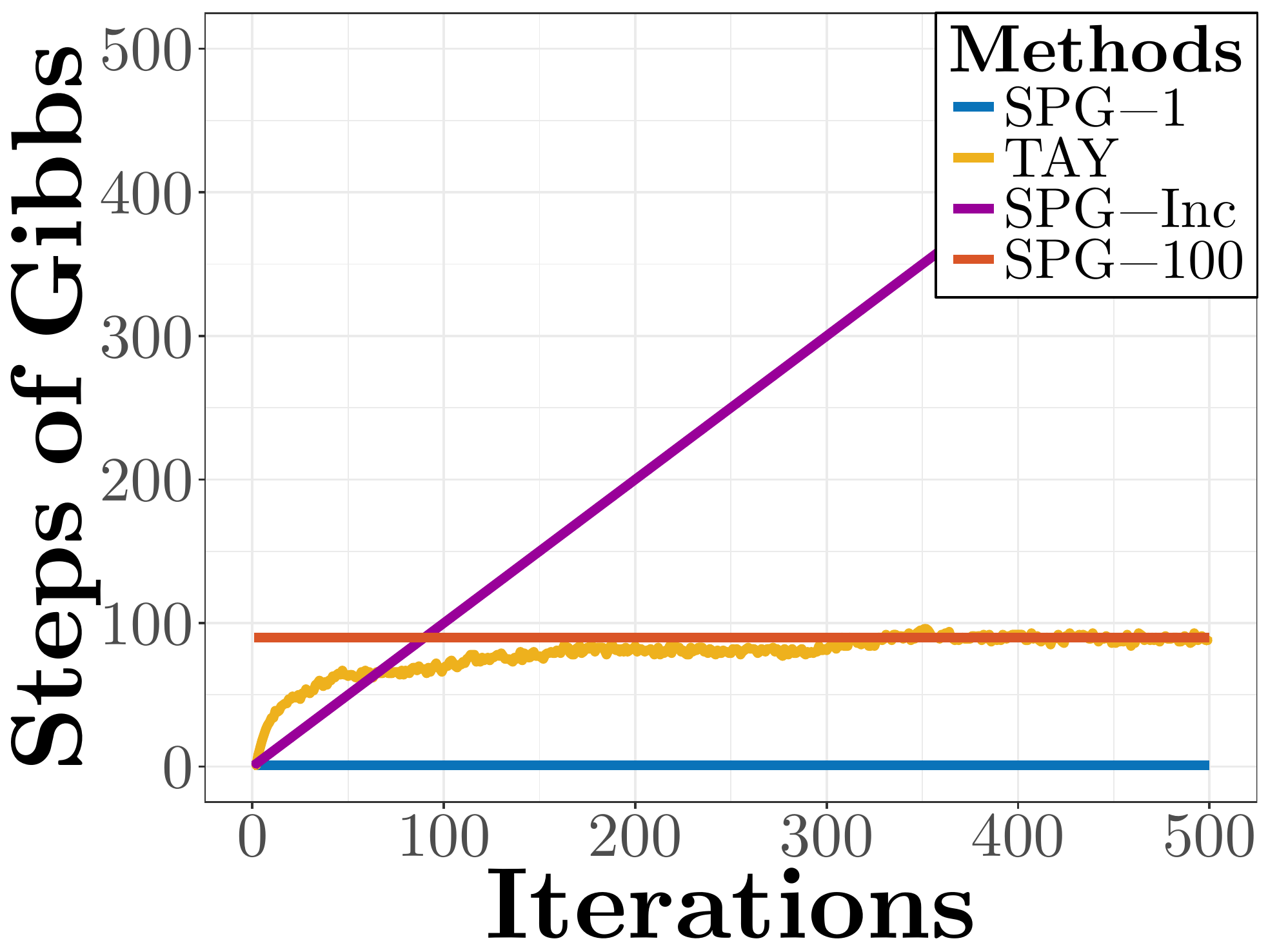}
\caption{$\tau$ v.s. Iterations\\$\lambda = 0.02$}
\end{subfigure}
\caption{Area under curve (AUC) and the steps of Gibbs sampling ($\tau$) for the structure learning of a 20-node network with different $\lambda$'s.}
\label{fig:str2}
\end{figure*}

\clearpage
\subsection{Comparison with the Pseudo-likelihood Method}
We compare TYA with the pseudo-likelihood method (Pseudo) under the same parameter configuration introduced in Section~\ref{sec:structure-learning}. Note that the two methods achieve a comparable performance: Pseudo is slightly better with 10 nodes and TAY outperforms a little with 20 nodes. This is consistent with the theoretical result that the two inductive principles are both sparsistent. 
\begin{figure*}[h]
\centering
\begin{subfigure}{0.32\textwidth}
\centering
\includegraphics[scale=0.35]{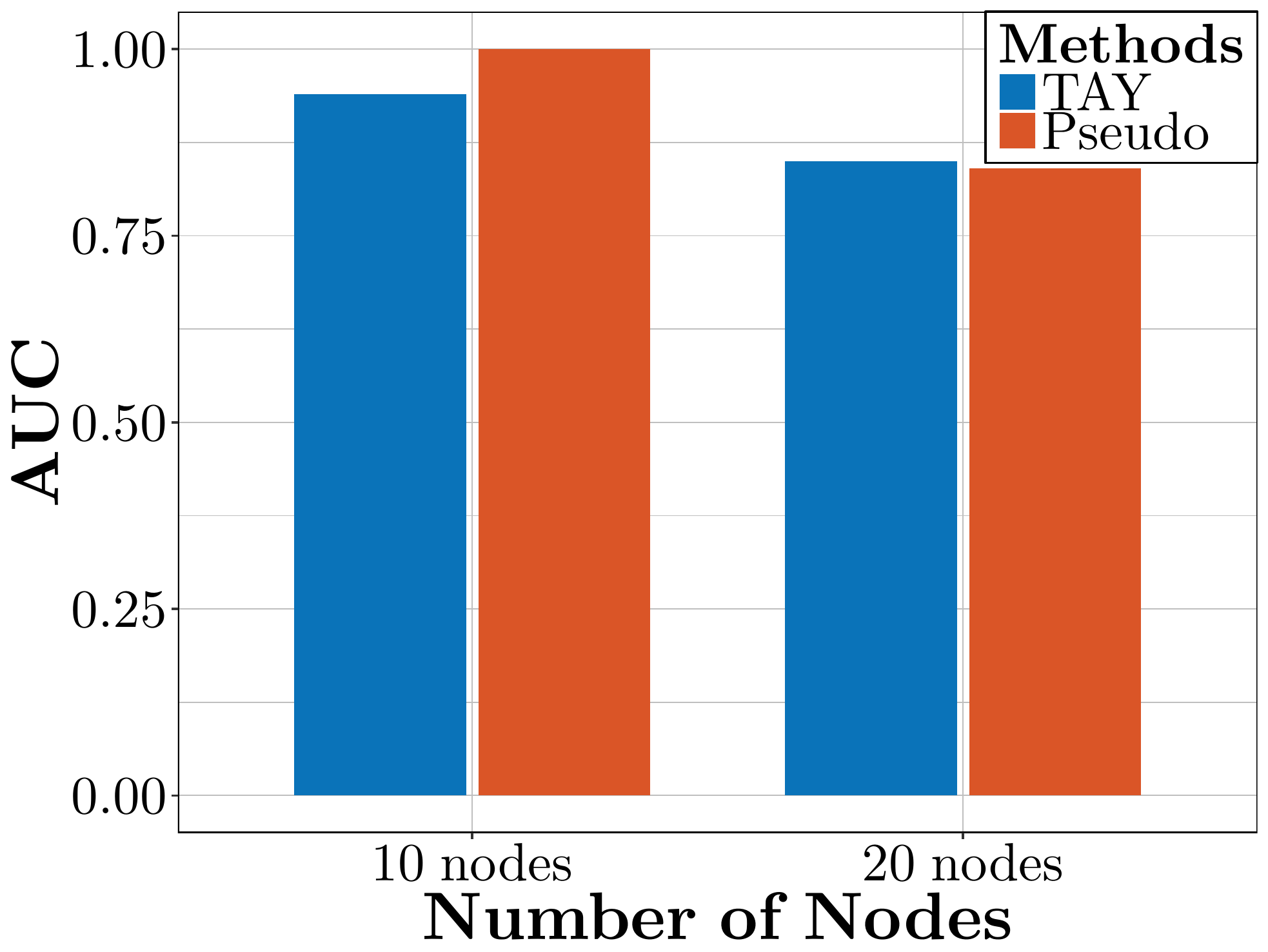}
\end{subfigure}
\caption{Area under curve (AUC) and for the structure learning of a 20-node network.}
\label{fig:str3}
\end{figure*}

\end{document}